\documentclass[pdflatex,sn-mathphys-ay]{sn-jnl}

\usepackage{graphicx}%
\usepackage{multirow}%
\usepackage{amsmath,amssymb,amsfonts}%
\usepackage{amsthm}%
\usepackage{mathrsfs}%
\usepackage{mathalpha}%
\usepackage[title]{appendix}%
\usepackage{xcolor}%
\usepackage{textcomp}%
\usepackage{manyfoot}%
\usepackage{lmodern}
\usepackage{wasysym}
\usepackage{booktabs}%
\usepackage{algorithm}%
\usepackage{algorithmicx}%
\usepackage{algpseudocode}%
\usepackage{listings}%
\usepackage{fullpage}
\usepackage{hyperref} 

\theoremstyle{thmstyleone}%
\newtheorem{theorem}{Theorem}
\theoremstyle{thmstyletwo}%
\newtheorem{remark}{Remark}%

\theoremstyle{thmstylethree}%
\newtheorem{definition}{Definition}%

\usepackage{amsmath,amsthm,amssymb,amsxtra,mathtools,bm}
\usepackage{array}
\usepackage{graphicx}
\usepackage{algorithm}
\usepackage{algpseudocode}
\usepackage{physics}
\usepackage{MnSymbol,wasysym}
\usepackage{tikzsymbols}
\usepackage{array,tabularx,multirow}
\usepackage{makecell}
\usepackage{array}
\usepackage{paralist}

\newcolumntype{P}[1]{>{\centering\arraybackslash}p{#1}}

\usepackage{verbatim}
\usepackage{enumerate}
\usepackage{parskip}
\usepackage{subfigure}

\usepackage[many]{tcolorbox}
\usetikzlibrary{decorations}

\usepackage{color,soul}
\definecolor{clemson-orange}{RGB}{234,106,32}
\definecolor{highlight-orange}{RGB}{255,150,150}
\definecolor{chicago-maroon}{RGB}{128,0,0}
\definecolor{cincinnati-red}{RGB}{190,0,0}
\definecolor{soft-cyan}{RGB}{68,85,90}
\definecolor{firebrick}{RGB}{178,34,34}
\definecolor{crimson}{RGB}{220,20,60}
\definecolor{cerrulean}{rgb}{0.165,0.322,0.745}
\definecolor{jaam}{rgb}{0.45,0.0,0.45}

\usepackage{hyperref}
\hypersetup{
  colorlinks   = true,    
  urlcolor     = jaam,    
  linkcolor    = firebrick,    
  citecolor    = blue      
}

\usepackage{parskip}

\usepackage{thmtools}
\declaretheoremstyle[
   headfont=\bfseries, 
   bodyfont=\normalfont\itshape, spaceabove=10pt,
   spacebelow=10pt]{mystyle}
\theoremstyle{mystyle}
\theoremstyle{definition}
\newtheorem*{theorem*}{Theorem}
\newtheorem{lemma}[theorem]{Lemma}

\newif\ifsolutions \solutionstrue

\def\final{0}
\newcommand{\reviewer}[3]{
  \expandafter\newcommand\csname #1\endcsname[1]{
    \ifthenelse{\equal{\final}{1}} {
      \textcolor{#3}{}
    } {
      \textcolor{#3}{\begin{center} \textbf{#2} ##1 \end{center}}
    }
  }
}

\reviewer{anirbit}{}{jaam}

\renewcommand{\ip}[2]{\left\langle#1,#2\right\rangle}

\def\1{\bm{1}}

\newcommand{\E}{\mathbb{E}}

\def\vepsilon{{\bm{\epsilon}}}

\def\vomega{{\bm{\omega}}}
\def\va{{\bm{a}}}

\def\vf{{\bm{f}}}
\def\vg{{\bm{g}}}

\def\vt{{\bm{t}}}
\def\vu{{\bm{u}}}
\def\vv{{\bm{v}}}
\def\vw{{\bm{w}}}
\def\vx{{\bm{x}}}
\def\vy{{\bm{y}}}
\def\vz{{\bm{z}}}

\def\mA{{\bm{A}}}

\def\mW{{\bm{W}}}

\def\mZ{{\bm{Z}}}

\DeclareMathAlphabet{\mathsfit}{\encodingdefault}{\sfdefault}{m}{sl}
\SetMathAlphabet{\mathsfit}{bold}{\encodingdefault}{\sfdefault}{bx}{n}

\def\gC{{\mathcal{C}}}
\def\gD{{\mathcal{D}}}

\def\gF{{\mathcal{F}}}

\def\gH{{\mathcal{H}}}

\def\gL{{\mathcal{L}}}

\def\gN{{\mathcal{N}}}
\def\gO{{\mathcal{O}}}

\def\gR{{\mathcal{R}}}
\def\gS{{\mathcal{S}}}

\newcommand{\bb}{\mathbb}

\newcommand{\R}{\bb R}

\newcommand{\cR}{\gR}

\begin{document}

\title[<short-title>]{Generalization Bounds for Physics-Informed Neural Networks for the Incompressible Navier-Stokes Equations \footnote{A preliminary version of this result was first presented at the meeting \href{https://drsciml.github.io/drsciml/}{DRSciML 2025 (link)}}}


\author[1]{\fnm{Sebastien} \sur{Andre-Sloan}}\email{sebastien.andre-sloan@postgrad.manchester.ac.uk}
\equalcont{These authors contributed equally to this work.}

\author[1]{\fnm{Dibyakanti} \sur{Kumar}}\email{dibyakanti.kumar@manchester.ac.uk}
\equalcont{These authors contributed equally to this work.}

\author[1]{\fnm{Alejandro} F \sur{Frangi}}\email{alejandro.frangi@manchester.ac.uk}

\author*[1]{\fnm{Anirbit} \sur{Mukherjee}}\email{anirbit.mukherjee@manchester.ac.uk}

\affil[1]{\orgdiv{Department of Computer Science}, \orgname{The University of Manchester},
\orgaddress{}}






\abstract{This work establishes rigorous first-of-its-kind upper bounds on the generalization error for the method of approximating solutions to the $(d+1)$-dimensional incompressible Navier-Stokes equations by training depth-2 neural networks trained via the unsupervised Physics-Informed Neural Network (PINN) framework. This is achieved by bounding the Rademacher complexity of the PINN risk. For appropriately weight bounded net classes our derived generalization bounds do not explicitly depend on the network width and our framework characterizes the generalization gap in terms of the fluid's kinematic viscosity and loss regularization parameters. In particular, the resulting sample complexity bounds are dimension-independent. Our generalization bounds suggest using novel activation functions for solving fluid dynamics. We provide empirical validation of the suggested activation functions and the corresponding bounds on a PINN setup solving the Taylor-Green vortex benchmark.}




\maketitle

\clearpage

\section{Introduction}\label{sec:introduction}

The foundational principles for nearly all natural sciences and engineering disciplines, predominantly manifest as Partial Differential Equations (PDEs). Hence, devising ways to find solutions of PDEs stands as a pivotal subject in the realm of computational science and engineering. Various techniques have been developed to solve PDEs ranging from the plethora of analytic methods to various discretization-based approaches. Significant progress has been made in providing analytic and numerical solutions to various specific PDEs --- but general methods do not exist. A major challenge is to approximately solve PDEs without facing the curse of dimensionality \cite{E_2021, han2018solving} which poses significant difficulties for even very structured classes of PDEs \cite{Beck_2021, berner2020numerically}. Recently, deep learning has emerged as a powerful tool for solving PDEs. Since some of the earliest attempts \cite{Lagaris_1998} at leveraging neural networks for PDE, deep learning techniques are being increasingly used for solving complex scientific problems, giving rise to the field now being known as ``AI for Science''. The potential that this direction holds can be glimpsed from the rapid developments in using AI for advancing climate science \cite{2021ComEE...2..159G, irrgang2021, Yuval_2020}.

A plethora of different ways \cite{karniadakis2021piml} have evolved in which one can write a loss function involving a partial differential equation (or its available approximate solutions) and neural net(s)  such that the net(s) obtained at the end of training can be used to infer approximate solutions to the intended PDEs. These data-driven methods of solving the PDEs can broadly be classified into two kinds, {\bf (1)} ones which train a single neural net to solve a specific PDE and {\bf (2)} operator methods, which train multiple nets in tandem to be able to solve a family of PDEs in one shot \cite{Lu_2021, LU2022114778, wang2021pdeOnet}. The operator methods are of particular  interest when the underlying physics is not known. While many kinds of methods have evolved that can solve the PDEs using deep learning, \cite{kaiser2021datadriven, erichson2019physicsinformed, Wandel_2021, li2022learning, salvi2022neural}, the Physics Informed Neural Networks (PINNs) \cite{raissi2019physics} have emerged as an umbrella framework, a general mathematical specification of which we briefly review in Appendix \ref{sec:pinnbackground}.


Motivated by their broad adoption, the theoretical foundations of PINNs are beginning to get rigorously established. A primary question in the theory of PINNs is to identify sufficient conditions under which minimizing the PINN risk guarantees convergence to the true PDE solution. To this end, in their pioneering paper \cite{mishra2022pinn} consider differential operators $\gL$ that are either linear or admit a well-defined linearization. Specifically, they suppose that there exists an operator $\Bar{\gL}$ such that, for all $u$ and $v$ in the appropriate function space, $\gL(u) - \gL(v) = \Bar{\gL}_{(u,v)}(u - v)$. Under the additional assumption that this linearization operator has a bounded inverse, it is shown that the functional distance between the predictor and the true PDE solution is bounded above by the PINN risk, up to multiplicative constants. This provided the first theoretical justification for why minimizing the PINN risk can lead to accurate PDE solutions. If we further assume periodic boundary conditions, the above results can be extended to the Navier-Stokes equations. But it is to be noted that this analysis is conducted using a pointwise weighted loss function, which differs from the standard losses typically used in practice.

The existence of neural networks that can simultaneously minimize both the PINN risk and the distance to the true PDE solution is investigated in \cite{tim22generic}. The authors show that for PDEs admitting a classical solution $u$, if there exists a neural network $\mathcal{N}$ that can approximate this solution to arbitrary accuracy, then there also exists a $\tanh$-activated neural network $\bar{\gN}$ such that its deviation from the true solution and its PINN risk can be made arbitrarily small, simultaneously. Furthermore, the depth of the surrogate network $\bar{\gN}$ scales as $\text{depth}(\bar{\gN}) = O(\text{depth}(\mathcal{N}))$, while the approximation accuracy of the PDE solution improves inversely with the width of $\bar{\gN}$. Additionally, \cite{tim22generic} provided an upper-bound on the generalization error, which explicitly depends on the number of parameters of the neural network. It is to be noted that the stability conditions assumed in \cite{mishra2022pinn} and the criteria for jointly minimizing the PINN risk and total error as described in \cite{tim22generic} are, a priori, distinct.

In \cite{zeinhofer24pinnerror}, the differential operator is assumed to be linear, and it is further assumed that there exists some $\alpha > 0$ such that, for all $u$, the functional norm of $\gL(u)$ is bounded below by the functional semi-norm of $\alpha \cdot u$. Under these assumptions, \cite{zeinhofer24pinnerror} proved that there exists a neural network capable of reducing the functional distance between the predictor and the true PDE solution.

For linear and second-order PDEs, where the coefficients of the function $u$ and its partial derivatives in $\gL(u)$ are $K$-Lipschitz and bounded in absolute value by $K$ for some $K>0$, \cite{zheyuan22xpinn} analyzed the Rademacher complexity of PINNs to derive an upper-bound on their worst-case generalization error, effectively providing a finite-sample complexity estimate. The derived bound applies to PINNs of arbitrary depth and width but this comes at the cost of restricting the PDE to being linear and second-order.

Extending PINN error analyses to non-linear PDEs, \cite{deryck23errorestpinn} examined the Navier-Stokes equations on a toroidal domain. In this setup, \cite{deryck23errorestpinn} proved that there exist PINNs capable of reducing the functional distance between the network predictor and the true PDE solution, thereby extending theoretical guarantees to a class of non-linear operators. In the next subsection we will compare our results against this. 








{\bf Navier-Stokes and Image Inpainting}
Image inpainting is the task of filling in missing data in an image. \cite{bertalmio2001nsinpaint} established an analogy between image processing and fluid dynamics by showing that image intensity can be interpreted as the stream function of a two-dimensional incompressible Navier-Stokes flow. Consequently, the inpainting problem can be approximated by solving the steady-state form of the $2D$ Navier-Stokes vorticity transport equation with small viscosity $\nu$,  
$
\vomega_t + \vv \cdot \nabla \vomega = \nu \Delta \vomega,
$
where the vorticity is defined as $\vomega \coloneqq \nabla \times \vv$, $\vv$ is the fluid velocity, and the stream function $\Psi$ is related to the vorticity through the Poisson equation $\Delta \Psi = \vomega$. The velocity field can then be obtained from the stream function as $\vv = \left[\pdv{\Psi}{y}, -\pdv{\Psi}{x}\right]$.


Extensions of the Navier-Stokes model have also been employed. For instance, \cite{ebrahimi2012nsvinpaint} applied the Navier-Stokes-Voight model of viscoelastic fluid for image inpainting, showing empirical improvements in both speed and quality.

Thus we note that even well beyond fluid dynamics, there are diverse reasons for keen interest in understanding numerical approximations to the Navier-Stokes PDE. 

\subsection{Summary of Results}

In this work, we establish the following upper-bound on the generalization error for training depth-$2$ neural networks via the unsupervised PINN loss corresponding to the $(d+1)$-dimensional Navier–Stokes equations. An informal restatement of the main theorem is provided below.

\begin{theorem*}[{\bf Informal Statement of Theorem~\ref{thm:NSRad}}]
    Consider neural networks $\gN_\vw$ belonging to a class of depth-2 neural networks. For the PINN empirical risk corresponding to the $(d+1)$-dimensional Navier–Stokes equations, denoted by $\hat{R}(\gN_\vw, \gS_n)$, where $\gS_n$ represents the training set, and the corresponding population risk $R(\gN_\vw)$, the data-averaged worst case generalization error --- defined as the difference between the empirical and population risks --- can be bounded as,
    \begin{align}
        &\E_{\gS_n} \left[\sup_{\gN_\vw} \left( \hat{R}(\gN_\vw, \gS_n) - R(\gN_\vw) \right)\right] \leq \frac{C_r}{\sqrt{N_r}} + \frac{C_0}{\sqrt{N_0}},
    \end{align}
    where $N_r$ denotes the number of training data points in the domain and $N_0$ the number of boundary points corresponding to the initial condition. The constants $C_r$ and $C_0$ can be expressed explicitly in terms of the Lipschitz constant of the loss function, the bounds on certain norms of the weights, the viscosity coefficient, the regularization on the loss terms enforcing the initial-condition and the incompressibility condition, and the bounds on the moments of the data distribution. The constant $C_r$ grows linearly with the viscosity coefficient.
\end{theorem*}

The above bound is independent of the distribution of the training data. Unlike \cite{deryck23errorestpinn}, where the upper-bounds are derived for the $\ell_2$ distance between an arbitrary neural net for any depth and the true solution, our analysis bounds the generalization error for solving the Navier-Stokes PDE. Moreover, we establish finite-sample complexity bounds that do not explicitly depend on the size of the neural network. While \cite{deryck23errorestpinn} derives bounds for a weighted PINN loss that is not typically used in practice, we provide guarantees for the standard PINN loss that is getting ubiquitously deployed. Lastly, Lemma \ref{lem:NSRad} shows that our sample complexity bound is dimension-independent, unlike corresponding estimates in Corollary 3.14 of \cite{deryck23errorestpinn}.




The above theorem applies to depth-$2$ neural networks for a variety of choices of activation functions such as $\tanh^k(x)$, $\text{sigmoid}^k(x)$ for $k \geq 1$and $\exp(-x)\cdot\mathrm{ReLU}^k(x)$ for $k \geq 3$.




\subsubsection{Organization}

We have organized this work as follows. In Section~\ref{sec:related}, we survey existing research on solving the Navier–Stokes equations using neural networks. In Section~\ref{sec:Radmotivation} we outline the motivations for studying Rademacher complexity bounds as a way to understand machine learning. In Section~\ref{sec:NSmathsetup}, we introduce the mathematical definitions and notations to be used in our work. Section~\ref{sec:NSmainresult} presents our main results, followed by Section~\ref{sec:NSproof}, which provides the full proof of the main theorem. The proofs of auxiliary lemmas are detailed in Section~\ref{sec:NSlemproof}. In Section~\ref{sec:num_experiments} we provide empirical validation of the main theorem. Finally, we conclude in Section~\ref{sec:conclusion} with remarks on the implications of our findings and directions for future work.



\subsection{Empirical Evidence of Solving the Navier-Stokes PDE via Deep-Learning} \label{sec:related}


An enormous amount of literature has emerged on solving the Navier-Stokes equations using neural networks. In this section, we first discuss some of the works employing fully unsupervised and semi-supervised loss functions, followed by a brief review of works that use fully supervised losses.

\subsubsection{Unsupervised Ways of Solving Navier-Stokes with Neural Networks}

Some of the earliest attempts of solving the $2D$ Navier-Stokes equations using neural networks were made by \cite{baymani2015annforns}, who performed unsupervised training of a three-layer neural network to approximate the stream function $\Psi$ of an incompressible laminar flow. The stream function was expressed as the sum of two functions, $A$ and $F$, where $A$ is independent of the neural network and enforces the boundary conditions, which is reminiscent of \cite{lagaris98pinn}. In this work, the net was trained to only learn the velocity and not the pressure, which can both be obtained from the learnt stream function.

More recently, \cite{jin2021nsfnet} trained neural networks in an unsupervised manner to solve the Navier-Stokes equations using two approaches: one in which the networks predict velocity and pressure, and another in which they predict vorticity and velocity. They demonstrated experimentally that both approaches can successfully solve three different cases: Kovasznay flow, two-dimensional cylinder wake, and three-dimensional Beltrami flow.

Extending unsupervised neural PDE solvers to turbulent regimes, \cite{vinuesa2022pinnrans} employed neural networks to solve the Reynolds-averaged Navier-Stokes (RANS) equations for incompressible turbulent flows in four different cases. More recently, \cite{wang2025simturbpinn} introduced PirateNet, a physics-informed neural network architecture designed for turbulent flow simulation. They demonstrated its effectiveness under the unsupervised setting on several benchmark problems, including $2D$ Kolmogorov flow, $3D$ Taylor-Green vortex, and $3D$ turbulent channel flow, showing that the method can capture complex flow dynamics across different regimes.

\subsubsection{Semi-supervised Ways of Solving Navier-Stokes with Neural Networks}

Semi-supervised neural network formulations refer to losses where the net is not only penalized for not satisfying the Navier-Stokes but is also penalized for not matching the observed data. In this direction, \cite{vijay2022pinn2dturb} trained neural networks to model two-dimensional turbulence within a periodic box domain, primarily using the semi-supervised PINN loss introduced in \cite{raissi2019pinn}. Furthermore, \cite{vijay2022pinn2dturb} proposes a dual-network architecture in which the total loss is partitioned between two sub-networks: one specialized for capturing high-wavenumber and the other focused on low-wavenumber structures.

A complementary line of work is presented in \cite{yusuf2024turbpinn}, where PINNs are employed to reconstruct the mean flow field from sparse mean velocity measurements in turbulent regimes. In their formulation, the authors augment the semi-supervised PINN loss function with additional physics-based terms derived from the Spalart-Allmaras (SA) turbulence model. By embedding the SA into the learning objective, the approach effectively constraints the solution space, thereby reducing the number of undetermined components in the mean-flow reconstruction. Empirical evaluations on the turbulent periodic hill benchmark demonstrated that the proposed framework yields more accurate velocity field approximations compared to baseline PINN formulations that do not incorporate turbulence-model information.

\subsubsection{Fully Data-Driven Ways of Solving Navier-Stokes with Neural Networks}

In the data assimilation setting, \cite{wang2020pidlforturb} trained neural networks in a supervised manner to predict turbulent flows; that is, given $10$ initial frames, the model predicts the spatiotemporal velocity fields for $60$ time steps ahead. Inspired by techniques such as Reynolds-averaged Navier-Stokes (RANS) and Large Eddy Simulation (LES), \cite{wang2020pidlforturb} decomposed the input velocity field into three components of different scales using two scale-separation operators, implemented as neural network-based filters. A different architectural perspective is explored in \cite{yang2024enhgraphunets}, which conducted a systematic investigation of the Graph U-Net architecture originally proposed by \cite{gao2019graphunet}. \cite{yang2024enhgraphunets} performed extensive experiments on spatio-temporal flow prediction tasks, specifically focusing on modeling vortex shedding dynamics behind a two-dimensional circular cylinder. 

Rather than replacing traditional solvers, some data-driven methods aim to enhance classical turbulence models. In this spirit, \cite{shojaee2025imprans} enhanced the accuracy of RANS-based simulations using neural networks. While RANS-based methods enable the simulation of flows at high Reynolds numbers and are faster, they are generally less accurate than LES-based approaches. Hence, \cite{shojaee2025imprans} aimed to achieve accuracy comparable to that of LES at high Reynolds numbers using neural networks in conjunction with RANS-based simulations.

Supervised learning has also been applied to super-resolution tasks in fluid dynamics. \cite{page2025supresturb} employs a convolution-based network to reconstruct high-resolution two-dimensional turbulent Kolmogorov flows from their low-resolution counterparts through fully supervised training. Building on the same objective in a higher-dimensional setting, \cite{barwey2025superrgnn} introduces a graph neural network architecture for mesh-based super-resolution of three-dimensional fluid flows and demonstrated its effectiveness on the Taylor-Green vortex benchmark.

Finally, \cite{alkin2025abupt} introduced a neural operator framework called the Anchored-Branched Universal Physics Transformer (AB-UPT), designed to encode diverse input representations and generate outputs at arbitrary spatial resolutions. In particular, \cite{alkin2025abupt} demonstrates the effectiveness of AB-UPT in modeling and solving turbulent flow problems.

\subsection{Motivations for Studying  the Rademacher Complexity}\label{sec:Radmotivation}

For a given class of predictors $\gF$, a loss function $\ell : \R^k \times \R^k \to [0, \infty)$, and a training dataset $\gS_n = \{ (\vx_i, \vy_i)_{i=1}^n\}$ sampled from a distribution $\gD$ on $\R^d \times \R^k$ the corresponding average $n-$point Rademacher complexity is defined as,
\begin{align*}
    \gR_n(\gF) \coloneqq \E_{\gS_n}\left[\E_\epsilon \left[ \sup_{f \in \gF} \frac{1}{n} \sum_{i=1}^n \epsilon_i \ell(\vy_i, f(\vx_i)) \right]\right],
\end{align*}
where each $\epsilon_i$ is uniformly sampled from $\{-1,1\}$.

In their seminal work, \cite{bartlett2003rademacher} demonstrated how Rademacher complexity can be used to upper-bound the data-averaged worst-case generalization error for any class of functions,
\begin{align}
\E_{\gS_n}\left[\sup_{f \in \gF} \left(\hat{R}(f,\gS_n) - R(f)\right)\right] \leq 2 \gR_n(\gF),
\end{align}

where $R(f) \coloneqq \E_{(\vx,\vy)}[\ell(\vy,f(\vx))]$ is the population risk, with $(\vx, \vy) \in \R^d \times \R^k$ is data sampled from a distribution $\gD$ on $\R^d \times \R^k$, and $\hat{R}(f,\gS_n) \coloneqq \frac{1}{n} \sum_{i=1}^n [\ell(\vy_i,f(\vx_i))]$ is the empirical risk. Alternative formulations of this bound can be found in \cite{shai2014book,tengyucs229m,telgarsky2021notes}.

Since the goal of learning is to minimize $R(f)$, and we only have access to $\hat{R}(f,\gS_n)$, one must ensure that the generalization gap, i.e., the difference between these two quantities, remains small. Hence successful learning depends on a delicate interplay between the size of the training set, the underlying data distribution, and the choice of the model class --- like the choice of the neural network architecture and hence the number of training parameters. But we note that it is non-trivial to establish $\hat{R}(f, \gS_n)$ converges to $R(f)$ as the number of samples increases.

In particular, given a class of predictors $\gF$, it can be prohibitively challenging to guarantee that the generalization error is uniformly small for all $f \in \gF$ and that it holds with high probability over training samples. Therefore a more tractable objective is to study the expected worst-case generalization error, defined as $\E_{\gS_n}[\sup_{f \in \gF} (\hat{R}(f,\gS_n) - R(f))]$. This quantity is independent of any specific learning algorithm or dataset but provides valuable insight into how the sample size, data distribution, and model complexity interact in determining generalization performance.


Following the work of \cite{bartlett2003rademacher}, several studies have established Rademacher-based upper-bounds for various classes of neural networks. \cite{neyshabur2015normbased} derived Rademacher-based bounds for ReLU neural nets for binary classification task under two types of regularization: bounding the norm of each weight individually and bounding the overall norm across all weights. \cite{bartlett2017spectrally} derived upper-bounds for ReLU nets on multi-class classification tasks under the ramp loss function. The analysis in \cite{bartlett2017spectrally} employs covering numbers to bound the Rademacher complexity of this function class. \cite{neyshabur2018the} focuses on analyzing the behavior of neural networks under over-parameterization. To this end, they derived Rademacher-based bounds for ReLU-activated, depth-2 neural networks on multi-class classification tasks with the ramp loss. Empirically, it was shown that these bounds correlate with the test error and decrease as the number of hidden units increases, thereby explaining the regularizing effect of over-parameterization. \cite{golowich2020sizeindp} showed that for ReLU-activated neural networks deployed for classification tasks with the ramp loss, these Rademacher-based upper-bounds can be formulated to be independent of the network size, with the depth dependence being non-exponential.

More recently \cite{lu2022priori} considered a rescaled Softplus-activated, depth-2 neural network to approximate the ground-state eigenvalue of the Schr\"odinger operator, $\gH u = - \Delta u + V u = \lambda u$, where $\gH \coloneqq - \Delta + V$ denotes the Hamiltonian operator. It is defined on $\Omega = [0,1]^d$ with boundary $\partial \Omega$, subject to the Neumann boundary condition of $\pdv{u}{\nu} = 0$ on $\partial \Omega$. \cite{lu2022priori} analyzed the Rademacher complexity of a loss function formulation of this task to provide corresponding upper-bounds.

\cite{zheyuan22xpinn} analyzed the Rademacher complexity of PINNs, for linear and second-order PDEs, where the coefficients of the function $u$ and its partial derivatives in $\gL(u)$ are $K$-Lipschitz and bounded in absolute value by a constant.

In contrast to the existing literature, in this work, we derive a first-of-its-kind upper-bounds on the Rademacher complexity of the class of PINN loss functions for a non-linear PDE solving target, in particular that of the Navier-Stokes equations. 

\section{Mathematical Setup}\label{sec:NSmathsetup}

In this section, we introduce the notation and definitions for the neural net class, the loss function, and the PDE target on which we focus for the PINN learning task under consideration. We begin by recalling the Huber loss \citep{huber1964loss}, which serves as the loss function for training our neural network.

\begin{definition}[{\bf Huber Loss}]\label{def:huber}
    For some $\delta \geq 0$ the Huber loss is defined as
    \begin{align*}
        \ell_{H,\delta} (x) \coloneqq
        \begin{cases}
            \frac{1}{2} x^2 &\quad \mbox{for } \abs{x}\leq \delta \\
            \delta \cdot (\abs{x} - \frac{1}{2} \delta) &\quad \mbox{for } \abs{x} > \delta
        \end{cases}        
    \end{align*}
\end{definition}

\subsection{PINN Loss for \texorpdfstring{$(d+1)$}{d+1}-Navier-Stokes}

Consider the incompressible Navier-Stokes PDE in  $d+1$ dimensional space-time, corresponding to an initial condition described by $\vf_0$.\footnote{ Recall the following notation: 
\[
\partial_t
\begin{bmatrix}
 u_1\\u_2\\.\\.\\.\\u_d\\
\end{bmatrix}
=
\begin{bmatrix}
 \partial_t u_1\\ \partial_t u_2\\.\\.\\.\\\partial_t u_d\\
\end{bmatrix}
\ , \
(\sum_{i=1}^{d} u_i\partial_{x_i})
\begin{bmatrix}
 u_1\\u_2\\.\\.\\.\\u_d\\
\end{bmatrix}
=
\begin{bmatrix}
 \sum_{i=1}^{d} u_i\partial_{x_i} u_1\\ \ \sum_{i=1}^{d} u_i\partial_{x_i} u_2\\.\\.\\.\\ \sum_{i=1}^{d} u_i\partial_{x_i} u_d\\
\end{bmatrix}
\ , \
\nu
\begin{bmatrix}
 \nabla^{2}u_1\\ \nabla^{2}u_2\\.\\.\\.\\ \nabla^{2}u_d\\
\end{bmatrix}
=
\nu
\begin{bmatrix}
 \sum_{i=1}^{d} \partial_{x_i}^2 u_1\\ \ \sum_{i=1}^{d} \partial_{x_i}^2 u_2\\.\\.\\.\\ \sum_{i=1}^{d} \partial_{x_i}^2 u_d\\
\end{bmatrix}
\]
}
\begin{align}\label{eq:dNS}
\nonumber \partial_t \vu + (\vu \cdot \nabla)\vu + \nabla p &= \nu \nabla^2 \vu,  
\\
\vu(\vx, 0) &= \vf_0(\vx), 
\\
\nonumber \nabla \cdot \vu &= 0
\end{align}
where we choose density parameter $\rho = 1$ for simplicity. In above $\vu$ and $p$ are the velocity and pressure field of the fluid and $\nu >0$ is the kinematic viscosity.

\begin{definition}[{\bf Neural Network for $(d+1)$-Navier-Stokes}]\label{def:nn}
    Let $\sigma : \R \to \R$, the input $\vz=(\vx,t) \in \R^{d+1}$, the weights $\mW \in \R^{p \times (d+1)}$, where $\mW = [ \vw_1, \vw_2, \dots, \vw_p ]^\top$, $\vw_q \in \R^{d+1}$ for $q\in\{1,\dots p\}$, and the frozen second layer weights $\mA_1\in\R^{d\times p}$ and $\va_2\in\R^p$. Then we can define a neural network $\gN_\vw : [0,1]^{d+1} \rightarrow \R^{d+1}$ whose output coordinates are denoted as mapping, 
    \[ (\vx,t) = \vz \mapsto (\gN_{\vw,u}(\vz), \gN_{\vw,p}(\vz) )\]
    \[\text{where,}\] 
    \begin{align*}
        \gN_{\vw,u}(\vz) \coloneqq \mA_1\sigma(\mW \vz),  ~\gN_{\vw,p}(\vz) \coloneqq \ip{\va_2}{\sigma(\mW \vz)}
    \end{align*}
    \[\text{and}  ~\sigma(\mW \vz) = [\sigma(\vw^\top_1 \vz), \sigma(\vw^\top_2 \vz),\dots,\sigma(\vw^\top_p \vz)]^\top\].

    
\end{definition}



\begin{definition}[{\bf Collocation Points for $(d+1)$-Navier-Stokes}]\label{def:collocation}
Let the sampled collocation points available to solve the PDE be denoted as the following two sets of points,

\begin{align*}
\{ (\vx_{ri},t_{ti}) \in [0,1]^d \times [0,1] \mid i=1,\ldots,N_r \}, ~\{ \vx_{0j} \in [0,1]^d  \mid j=1,\ldots,N_0 \}
\end{align*}

We also assume that there exists constant such that $\E_{\vz_{0j}}[\norm{\vz_{0j}}_2^2] \leq C_{\vz_0}^2, \forall j \in [N_0]$, where $\vz_{0j} \coloneqq (\vx_{0j},0)$, and $\E_{\vz_i}[\norm{\vz_i}_2^2] \leq C_{\vz}^2, \forall i \in [N_r]$, where $\vz_i = (\vx_{ri}, t_{ri})$.
\end{definition}

\begin{definition}[{\bf Empirical Risk for $(d+1)$-Navier-Stokes}]\label{def:erisk}
For a Lipschitz loss function $\ell : \R \rightarrow [0,\infty)$, a neural net $\gN_\vw : [0,1]^{d+1} \rightarrow \R^{d+1}$ and regularization parameters $\lambda_0, \lambda_1 >0$, the corresponding empirical risk with collocation points given by Definition \ref{def:collocation} can be written as,

\begin{align}
     \hat{R}(\gN_\vw) \coloneqq &\frac{1}{N_r}\sum_{i=1}^{N_r} \sum_{k=1}^d\ell \left( (\partial_t \gN_{\vw,u})_k + ((\gN_{\vw,u} \cdot \nabla) \gN_{\vw,u})_k + (\nabla \gN_{\vw,p})_k - \nu (\nabla^2 \gN_{\vw,u})_k \right )\mid_{(\vx_{ri},t_{ri})} \nonumber\\
     &+ \frac{\lambda_0}{N_r}\sum_{i=1}^{N_r}\ell \left  ( \nabla\cdot \gN_{\vw,u} \right )\mid_{(\vx_{ri},t_{ri})}   + \frac{\lambda_1}{N_0}\sum_{j=1}^{N_0} \sum_{k=1}^d\ell \left ( (\gN_{\vw,u}(\vx,0))_k - (\vf_0 (\vx))_k \right )\mid_{\vx_{0j}}
\end{align}
\end{definition}

The three terms mentioned above individually impose penalties when the net does not fulfil the three constraints defined by the PDE outlined in equation \ref{eq:dNS} at the specified collocation points.

\begin{definition}[{\bf Loss Function for $(d+1)$-Navier-Stokes}]\label{def:lossclass}
    We can define two separate loss functions, $\ell_{res}$ and $\ell_{0}$ s.t they measure the error at any space-time interior point and a spatial point, respectively, as follows,
    \begin{align*}
        \ell_{res}(\gN_\vw, (\vx_{ri},t_{ri})) &\coloneqq \sum_{k=1}^d\ell \left( (\partial_t \gN_{\vw,u})_k + ((\gN_{\vw,u} \cdot\nabla)\gN_{\vw,u})_k + (\nabla \gN_{\vw,p})_k - \nu (\nabla^2 \gN_{\vw,u})_k \right )\mid_{(\vx_{ri},t_{ri})} \\
        &\hspace{2.00em}+ \lambda_0 \cdot \ell \left  (\grad \cdot \gN_{\vw,u} \right )\mid_{(\vx_{ri},t_{ri})}\\
        \ell_{0}(\gN_\vw, \vx_{0j}) &\coloneqq \lambda_1 \sum_{k=1}^d \ell \left ( (\gN_{\vw,u}(\vx,0))_k - (\vf_0 (\vx))_k \right )\mid_{\vx_{0j}}
    \end{align*}
    and thus the empirical risk can we written as,
    \begin{align}\label{eq:genpinn}
        \hat{R}(\gN_\vw) \coloneqq &\frac{1}{N_r}\sum_{i=1}^{N_r} \ell_{res}(\gN_\vw, (\vx_{ri},t_{ri}))  + \frac{1}{N_0}\sum_{j=1}^{N_0} \ell_{0}(\gN_\vw, \vx_{0j}).
    \end{align}
\end{definition}
Towards proving Theorem~\ref{thm:NSRad}, we will define certain classes of weights that satisfy bounds on particular combinations of these weights.

\begin{definition}[{\bf Class of Weights for $(d+1)$-Navier-Stokes}]\label{def:Wspace}
    Given weights $\mW \in \R^{p \times (d+1)}$, let 
    $\vw_t = \mW_{:,\,d+1}$, $\vw_{\vx_m} = \mW_{:,\,m}$, and $\va_{1k} = \mA_{1k,:}$ for $m, k \in \{1, \dots, d\}$. Define $\vw_\vx = \sum_{m=1}^d \vw_{\vx_m}$ and $\va_1 = \sum_{k=1}^d \va_{1k}$. 
    
    We can then define the following functions: \begin{inparaenum}[\bf (i)]
        \item $f_1(\mW)=\va_1\odot\vw_t$, 
        \item $f_{2m}(\mW)=\va_{1}\odot\vw_{\vx_m}$,
        \item $f_3(\mW)=\va_2\odot\vw_\vx$,
        \item $f_4(\mW)=\sum_{m=1}^d\va_1\odot\vw_{\vx_m}\odot\vw_{\vx_m}$, and
        \item $f_5(\mW)=\sum_{m=1}^d\va_{1m}\odot\vw_{\vx_m}$,
    \end{inparaenum}
    where $\va_{1m}$, $\va_1$, $\va_2$, $\vw_{\vx_m}$, $\vw_t$, and $\vw_{\vx}$ are as given above.
    Now, given constants $B_{f_i}$ for $i \in [5]$ by the following constraints are defined,
    \begin{gather}\label{eq:fbounds}
        \sum_{q=1}^p \abs{f_s(\mW)_q} \leq B_{f_s} \textrm{ for $s \in \{1,3,4,5\}$}\ ; \qquad
        \sum_{m=1}^d\sum_{q_1,q_2 = 1}^p \abs{f_{2m}(\mW)_{q_1}} \abs{a_{1m,q_2}} \leq B_{f_2},
    \end{gather}
    where $f_1$, $f_{2m}$, $f_3$, $f_4$ and $f_5$ are the functions defined above.
    Corresponding to the above, we define the class of weights $\gC$ as,
    \begin{align*}
        \gC \coloneq \big\{\mW \in \R^{p \times (d+1)}\ \big| \mW \text{ satisfies the bounds given in Equation \ref{eq:fbounds}} \big\}.
    \end{align*}
    
    We further define $\textrm{RS}(\gC)$ as the set of possible row vectors from the elements of $\gC$.
\end{definition}

The neural net chosen in Definition \ref{def:nn} has only one layer of activation and a single layer of trainable weights. However, when this model is trained using the loss function given in Definition \ref{def:lossclass}, with activation $\sigma(\cdot)=\tanh^3(\cdot)$ — which satisfies the conditions to be specified in Theorem \ref{thm:NSRad} — at $N_r=1000$, $N_0=2500$, $\lambda_0=1$, and $\lambda_1=0.3$ on the domain $(x,y,t)\in[0,2]^2\times[0,1]$, we can see in Figure \ref{fig:TG_plot1} that it can indeed appreciably learn the classically challenging benchmark Taylor-Green vortex solutions \citep{taylor1937mechanism} of Navier-Stokes.

We posit that the level of performance shown in  Figure \ref{fig:TG_plot1} is substantial despite the net being of depth $2$, and hence, this motivates a detailed mathematical investigation of the learning capacity of this setup.

\begin{figure}[htbp!]
\begin{subfigure}
    \centering
    \includegraphics[width=0.48\linewidth]{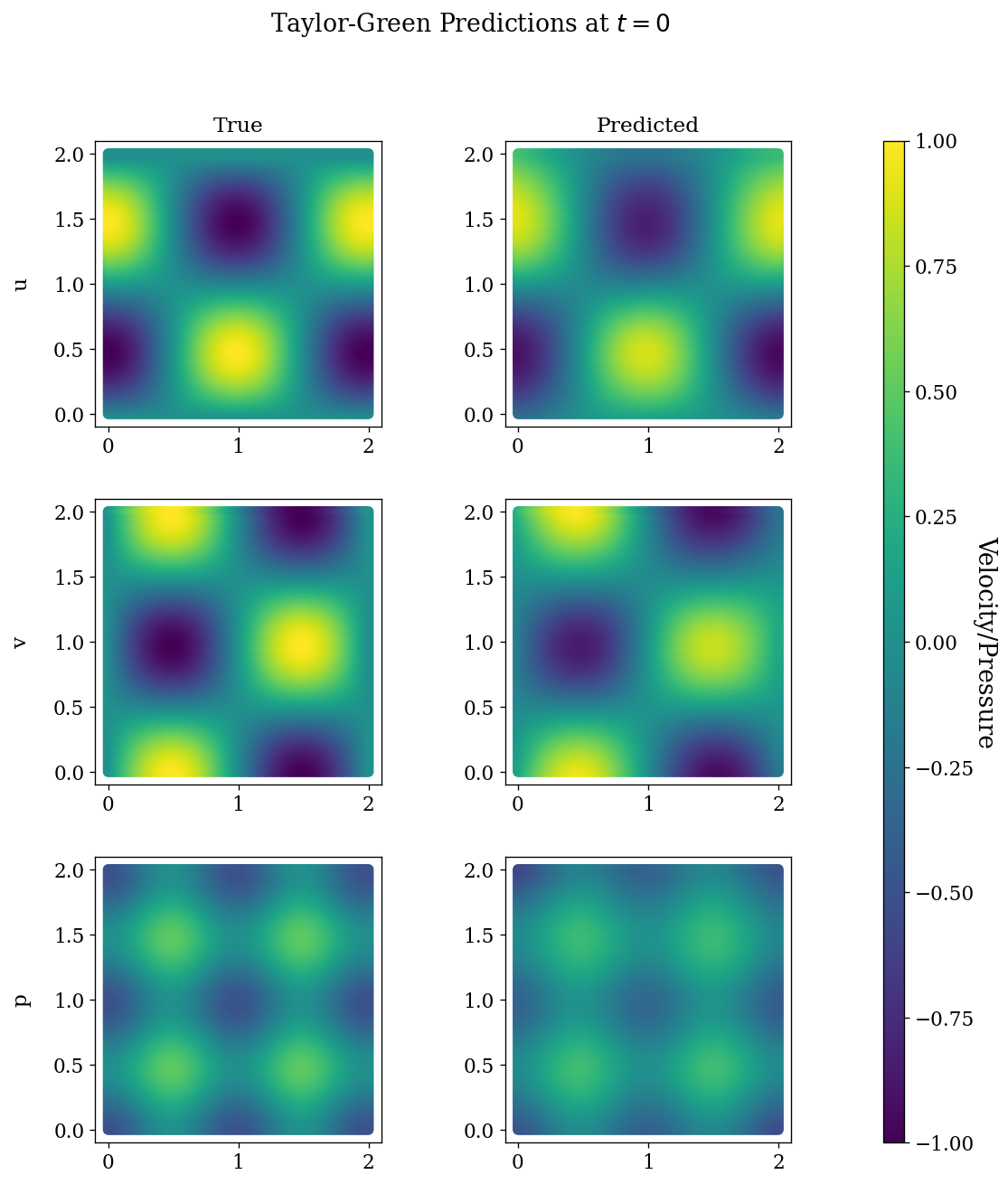}
\end{subfigure}
\hfill  
\begin{subfigure}
    \centering
    \includegraphics[width=0.48\linewidth]{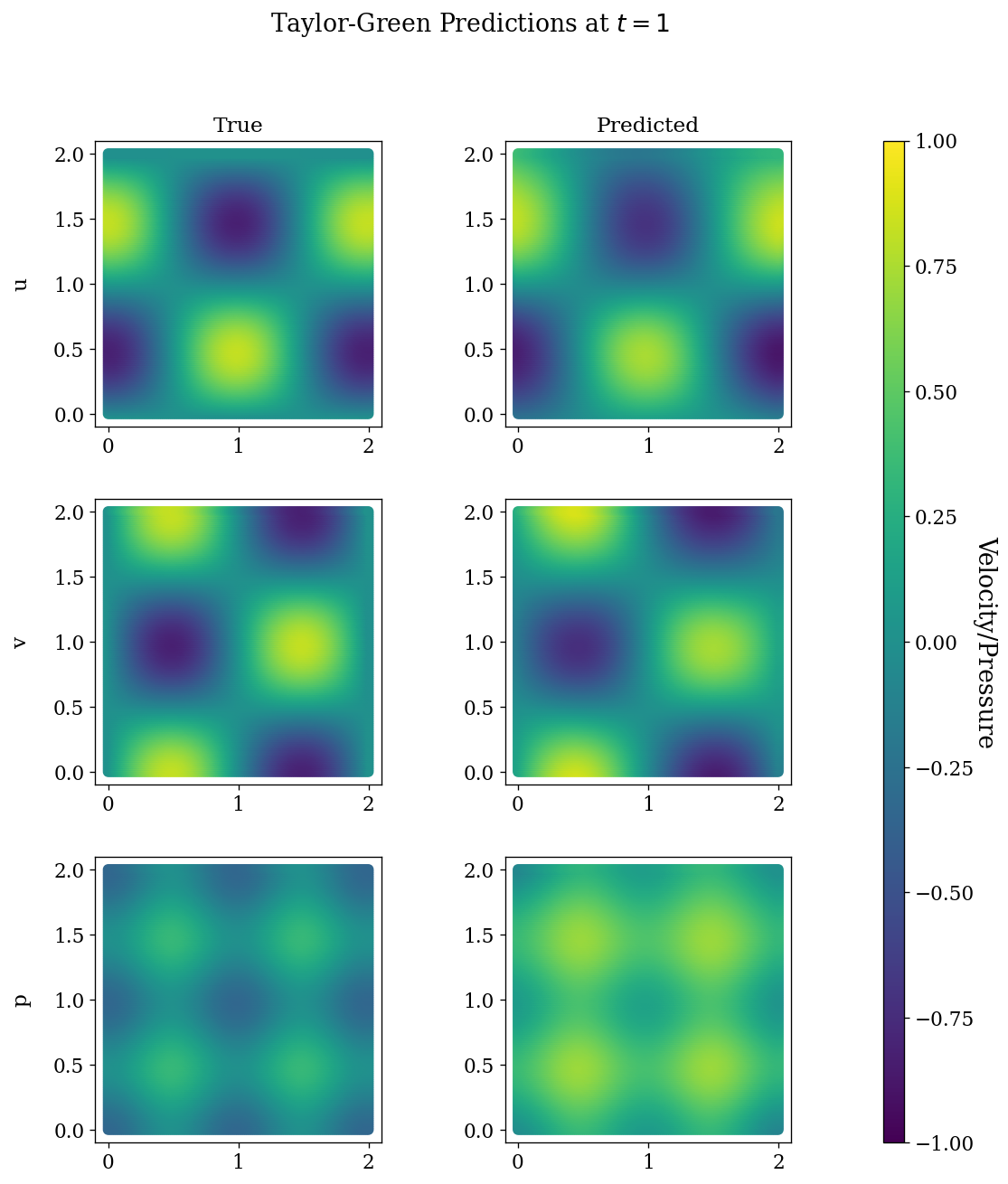}
\end{subfigure}
\caption{The left column is a snapshot of the true solution ($t=0$ for the left plot and $t=1$ for the right plot) and the right column is the neural net prediction. The rows represent the three components of the solution, $u,v$, and $p$. The domain is $(x,y)\in[0,2]^2$, with the viscosity set to $10^{-2}$ and the density set to $1$.}
\label{fig:TG_plot1}
\end{figure}

In Section \ref{sec:num_experiments}, we will return to the demonstration above to study it quantitatively in the light of the generalization bound we present below.





\section[Main Result on Generalization Error for (d+1)-Navier-Stokes]{Main Result on Generalization Error for $(d+1)$-Navier-Stokes}\label{sec:NSmainresult}







In this section, we provide a Rademacher-based generalization upper bound for solving the Navier–Stokes equations with neural networks using a physics-informed loss function.

\begin{theorem}[{\bf Rademacher Based Upper-bound for $(d+1)$-Navier-Stokes}]\label{thm:NSRad}
    Let $\gN_{\vw}$ be a neural network as in Definition~\ref{def:nn}, with weights $\mW \in \gC$, where the class of weights $\gC$ is defined in Definition~\ref{def:Wspace}. We make the following assumptions about the activation function $\sigma(x)$:
    \begin{inparaenum}[(i)]
    \item $\sigma(x)$ is twice differentiable,
    \item $\sigma(x)$, $\sigma'(x)$, and $\sigma''(x)$ are Lipschitz continuous with constants $L_{\sigma}$, $L_{\sigma'}$, and $L_{\sigma''}$, respectively, and
    \item $\sigma(x)$ and $\sigma'(x)$ are bounded by constants $B_{\sigma}$ and $B_{\sigma'}$, respectively.
    \end{inparaenum}
    Let the two components of the loss function, $\ell_{res}$ and $\ell_0$, be defined as in Definition~\ref{def:lossclass}. From Definition~\ref{def:collocation}, we recall the constants $C_\vz$ and $C_{\vz_0}$, which characterize the distributions of the collocation points in the interior and on the initial-condition boundaries, respectively. We can then bound the generalization error of the empirical loss by
    \begin{align}
        &\E_{\substack{(\vx_{ri},t_{ri}) \sim \gD_1 \forall i \in [N_r] \\ \vx_{0j} \sim \gD_2 \forall j\in[N_0]}} \Bigg[\sup_{\mW \in \gC} \biggl( \frac{1}{N_r} \sum_{i=i}^{N_r} \ell_{res}(\gN_\vw,(\vx_{ri},t_{ri})) + \frac{1}{N_0} \sum_{j=i}^{N_0} \ell_{0}(\gN_\vw,\vx_{0j})\nonumber\\ 
        &\hspace{15.00em}- \left( \E_{(\vx_r,t_r) \sim \gD_1}[\ell_{res}(\gN_{\vw},(\vx_{r},t_{r}))] + \E_{\vx_0 \sim \gD_2}[\ell_{0}(\gN_{\vw},\vx_{0})] \right) \biggr)\Bigg] \nonumber\\
        &\hspace{1.00em}\leq \frac{2 \delta (B_{\vw} C_\vz C_1 + C_2)}{\sqrt{N_r}} + \frac{4 \lambda_1 \delta B_\va (B_{\vw} C_{\vz_0} L_\sigma + \abs{c_0})}{\sqrt{N_0}},
    \end{align}
    
    here, $C_1 \coloneqq 2B_{f_1}L_{\sigma'}+4B_{f_2}(B_\sigma L_{\sigma'}+B_{\sigma'}L_\sigma)+2B_{f_3}L_{\sigma'}+2\nu B_{f_4} L_{\sigma''}+2 \lambda_0 B_{f_5}L_{\sigma'}$ and $C_2 \coloneqq  B_{f_1}\abs{c_1} + B_{f_2}(2B_{\sigma'}\abs{c_0} + \abs{c_0 c_1}) + B_{f_3}\abs{c_1} + \nu B_{f_4}\abs{c_2} + \lambda_0 B_{f_5}\abs{c_1}$, where $B_{f_s}$ for $s \in [5]$ are as defined in Definition~\ref{def:Wspace}, $c_0 = \sigma(0)$, $c_1 = \sigma'(0)$, $c_2 = \sigma''(0)$, and $\delta$ is the Lipschitz constant of the loss function $\ell_{H,\delta}$ as in Definition~\ref{def:huber}. We also assume that there exist constants $B_{\vw}$ and $B_\va$ such that $\norm{\vw_q}_2 \leq B_{\vw}, \forall q \in [p]$, where $\vw_q = \mW_{q,:}$, and $\sum_{q=1}^p \abs{(\va_1)_q} \leq B_\va$, for $\va_1$ is as in Definition~\ref{def:nn}.
\end{theorem}

The proof of the above theorem is given in Section~\ref{sec:NSproof}. 

We note that similar to the upperbound on total error that was obtained in \cite{deryck23errorestpinn} for a similar setup, here too the upper-bound increases linearly with the viscosity. Unlike \cite{deryck23errorestpinn}, our upper bound does not scale with the dimension $d$ of the fluid domain.

\begin{remark}
    The constants in Theorem~\ref{thm:NSRad} for specific choices of activation functions are as follows:
    \begin{enumerate}[(i)]
        \item $\sigma = \tanh$:
        \begin{equation} 
            L_\sigma = L_{\sigma'} = 1,\;\;
            L_\sigma'' = 2,\;\;
            B_\sigma = B_{\sigma'} = 1,\;\;
            c_0 = c_2 = 0,\;\;
            c_1 = 1.
        \end{equation}
        \item $\sigma = \tanh^3$:
        \begin{equation}
            L_\sigma = 0.75,\;\;
            L_{\sigma'} = 1.4,\;\;
            L_\sigma'' = 6,\;\;
            B_\sigma = 1,\;\;
            B_{\sigma'} = 0.75,\;\;
            c_0 = c_1 = c_2 = 0.
        \end{equation}
    \end{enumerate}
\end{remark}

In Section~\ref{sec:num_experiments} we show empirical results of our bound and its correlation to the generalization error.

\begin{lemma}\label{lem:NSRad}
    Continuing within the framework of Theorem~\ref{thm:NSRad}, it follows that a way to ensure that the RHS is $\gO(\epsilon)$ the number of collocation points $N_r$ and initial condition points $N_0$ can be chosen as,
    \begin{align*}
        N_r &= \gO\left(\frac{\delta^2 (B_\vw C_\vz C_1 + C_2)^2}{\epsilon^2}\right)\\
        N_0 &= \gO\left(\frac{\lambda_1^2 \delta^2 B_\va^2 (B_\vw C_{\vz_0} L_\sigma + \abs{c_0})^2}{\epsilon^2}\right),
    \end{align*}
    where $C_1 \coloneqq 2B_{f_1}L_{\sigma'}+4B_{f_2}(B_\sigma L_{\sigma'}+B_{\sigma'}L_\sigma)+2B_{f_3}L_{\sigma'}+2\nu B_{f_4} L_{\sigma''}+2 \lambda_0 B_{f_5}L_{\sigma'}$ and $C_2 \coloneqq  B_{f_1}\abs{c_1} + B_{f_2}(2B_{\sigma'}\abs{c_0} + \abs{c_0 c_1}) + B_{f_3}\abs{c_1} + \nu B_{f_4}\abs{c_2} + \lambda_0 B_{f_5}\abs{c_1}$.
\end{lemma}


\begin{remark}
    From the above lemma, we can further infer that a suggested ratio between the collocation points and initial points is, $$\frac{N_r}{N_0} = \gO\left(\frac{(B_\vw C_\vz C_1 + C_2)^2}{\lambda_1^2 B_\va^2 (B_\vw C_{\vz_0} L_\sigma + \abs{c_0})^2}\right)$$
\end{remark}



\section{Generalization Bound for Solving the Navier-Stokes PDE (Theorem \ref{thm:NSRad})}\label{sec:NSproof}

\subsection{An Outline of the Proof Technique}

Towards proving Theorem \ref{thm:NSRad} we need five key lemmas to be proven. Using Lemma~\ref{lem:genbound}, we express the data-averaged worst-case generalization error in terms of the Rademacher complexities of the residual and initial condition loss components. We then expand the Rademacher complexities of these components using Lemma~\ref{lem:partialderivatives}, making them explicit in terms of the network parameters and derivatives of the activation function. Next, by applying Lemmas~\ref{lem:dabsvalremtanh}–\ref{lem:dcontract2}, we reduce each term to the Rademacher complexity of a corresponding linear function class. Finally, invoking standard bounds on the Rademacher complexity of linear functions, we obtain the desired generalization bounds.

In the following lemma, we show that the Rademacher complexity of the full PINN empirical risk, as defined in Definition~\ref{def:lossclass}, can be upper bounded independently in terms of the residual and boundary losses.
\begin{lemma}\label{lem:genbound}
    Continuing in the setup of Definition \ref{def:lossclass}, it follows that,
    \begin{align*}
        &\E_{\substack{(\vx_{ri},t_{ri}) \sim \gD_1 \forall i \in [N_r] \\ \vx_{0j} \sim \gD_2 \forall j\in[N_0]}} \Bigg[\sup_{h\in\gH} \Bigg( \frac{1}{N_r} \sum_{i=i}^{N_r} \ell_{res}(h,(\vx_{ri},t_{ri})) + \frac{1}{N_0} \sum_{j=i}^{N_0} \ell_{0}(h,\vx_{0j}) \\
        &\hspace{15.00em}- \left( \E_{(\vx_r,t_r) \sim \gD_1}[\ell_{res}(h,(\vx_{r},t_{r}))] + \E_{\vx_0 \sim \gD_2}[\ell_{0}(h,\vx_{0})] \right) \Bigg)\Bigg]\\
        &\leq 2\gR_{res} (\gH) + 2\gR_{0} (\gH)
    \end{align*}
    where,
    \begin{align*}
        \gR_{res} (\gH) &\coloneqq \E_{\substack{(\vx_{ri},t_{ri}) \sim D_1 \forall i \in [N_r]}}\left[ \E_{\epsilon_i\sim \{\pm 1\}} \left[ \sup_{h\in\gH} \left(\frac{1}{N_r} \sum_{i=i}^{N_r} \epsilon_i\ell_{res}(h,(\vx_{ri},t_{ri})) \right) \right] \right] \\
        \gR_0 (\gH) &\coloneqq \E_{\substack{\vx_{0j}' \sim \gD_2 \forall j\in[N_0]}} \left[\E_{\vepsilon \sim \{\pm 1\}^{N_0}} \left[ \sup_{h\in\gH} \left( \frac{1}{N_0} \sum_{j=i}^{N_0} \epsilon_j \ell_{0}(h,\vx_{0j}) \right)\right]\right]
    \end{align*}
\end{lemma}

To express the residual loss in terms of its derivatives, we employ the following simplifications outlined below.

\begin{lemma}\label{lem:partialderivatives}
    The specific partial derivatives needed to expand the loss are given by:
    \begin{align*}
        \partial_{x_k} \gN_{\vw,p}(\vz) =\ip{\va_2\odot\vw_{x_k}}{\sigma'(\mW\vz)}, ~&(\partial_t\gN_{\vw,\vu}(\vz))_k =\ip{\va_{1k}\odot\vw_t}{\sigma'(\mW\vz)}\\
        \nabla\cdot\gN_{\vw,u}(\vz)=\sum_{m=1}^d\ip{\va_{1m}\odot\vw_{x_m}}{\sigma'(\mW\vz)},&((\gN_{\vw,\vu}\cdot\nabla)\gN_{\vw,\vu})_k =\sum_{m=1}^d(\ip{\va_{1m}}{\sigma(\mW\vz)})(\ip{\va_{1k}\odot\vw_{\vx_m}}{\sigma'(\mW\vz)})\\
        (\nabla^2\gN_{\vw,\vu}(\vz))_k&=\sum_{m=1}^d\ip{\va_{1k}\odot\vw_{x_m}\odot\vw_{x_m}}{\sigma''(\mW\vz)}
    \end{align*}
    where $\vz = (\vx, t)$.
\end{lemma}

  Next we state a modification of the contraction principle stated in Lemma 5.12 in \cite{tengyucs229m}, rendering it suitable for our analysis.

\begin{lemma}\label{lem:dabsvalremtanh}(A Variation of Lemma 5.12 of \cite{tengyucs229m})
    Given a real valued function $f_\vw$ parameterized by $\vw$ and a data-set $\mZ = (\vz_1,\ldots,\vz_n)$ define, $f_\vw(\mZ) = (f_\vw(\vz_1),\dots ,f_\vw(\vz_n))$ and $g_\vw(\mZ) = (f_\vw(\vz_1)-c,\dots ,f_\vw(\vz_n)-c)$, where $c \in \R$ is some constant. Suppose that for any $\vepsilon \in \{\pm 1\}^n$, $\sup_\vw \ip{\vepsilon}{g_\vw(\mZ)} \geq 0$. Then,
    \begin{align}
        \E_{\vepsilon\sim\{\pm 1\}^n}\left[ \sup_\vw \abs{\ip{\vepsilon}{f_\vw(\mZ)}} \right] \leq 2 \E_{\vepsilon\sim\{\pm 1\}^n}\left[ \sup_\vw \ip{\vepsilon}{g_\vw(\mZ)} \right] + \abs{c}\sqrt{n}
    \end{align}
\end{lemma}


We present two more contraction lemmas that will be instrumental in how we simplify the Rademacher complexity bounds in the subsequent analysis. The standard Talagrand's contraction (Eq. 4.20, \cite{talagrand1991Prob}), which applies to general Lipschitz functions, cannot be directly used for our loss function because the analysis in here leads to taking supremums of inner-products of pairs of functions which share variables. To address this, we introduce Lemma~\ref{lem:dcontract2} to handle the nonlinear component of the loss corresponding to the $(\vu \cdot \nabla)\vu$ term of the Navier-Stokes PDE defined in Equation~\ref{eq:dNS}, and Lemma~\ref{lem:dcontract1tanh} for the remaining components.



\begin{lemma}\label{lem:dcontract1tanh}
    Assume that $\phi : \R \rightarrow \R$ is an $L_\phi$-Lipschitz continuous, $\phi(0) = c$, for some constant $c \in \R$, and $\sum_{m=1}^{p} \abs{f(\mW)_m} \leq B$, where $f(\mW) : \R^{p \times d} \to \R^p$. Then,
    \begin{align*}
        \E_{\vepsilon \sim \{\pm 1\}^{N_r}} \left[ \sup_{\mW\in\gC} \left( \frac{1}{N_r} \sum_{i=1}^{N_r} \epsilon_i \ip{f(\mW)}{\phi(\mW \vz_i)} \right) \right] \leq \frac{2B L_{\phi}}{N_r} \E_{\vepsilon \sim \{\pm 1\}^{N_r}} \left[ \sup_{\bar{\vw}\in\textrm{RS}(\gC)} \sum_{i=1}^{N_r} \epsilon_i \ip{\bar{\vw}}{\vz_i} \right] + \frac{B\abs{c}}{\sqrt{N_r}}
    \end{align*}
\end{lemma}


\begin{lemma}\footnote{The LLM Claude was used in some parts of the proof of Lemma~\ref{lem:dcontract2}}\label{lem:dcontract2}
    Assume that $\phi_1 : \R \rightarrow \R$ and $\phi_2 : \R \to \R$ are Lipschitz continuous function with Lipschitz constants $L_{\phi_1}$ and $L_{\phi_2}$ respectively, and $\phi_1(0) \cdot \phi_2(0) = k$, for some constant $k \in \R$. Let $\phi_1 \leq B_{\phi_1}$, $\phi_2 \leq B_{\phi_2}$ and let there be $d$ functions $f_{m},m\in[d]$,$f_m(\mW) : \R^{p \times d} \to \R^p$ such that $\sum_{m=1}^d \sum_{q_1,q_2 = 1}^p \abs{f_m(\mW)_{q_1}} \abs{\va_{m,q_2}} \leq B$. Then, 
    \begin{align}
        &\E_{\substack{\vepsilon \sim \{\pm 1\}^{N_r}}} \left[ \sup_{\mW\in\gC} \left(\frac{1}{N_r} \sum_{i=1}^{N_r} \sum_{m=1}^d\epsilon_i \ip{f_m(\mW)}{\phi_1 (\mW \vz_i)} \ip{\va_m}{\phi_2(\mW \vz_i)} \right)\right] \nonumber\\
        &\leq \frac{4B(B_{\phi_1} L_{\phi_2} + B_{\phi_2} L_{\phi_1})}{N_r}\, \mathbb{E}_{\epsilon}\left[\sup_{\bar{\vw} \in \mathrm{RS}(C)} \sum_{i=1}^{N_r} \epsilon_i \langle \bar{\vw}, \vz_i\rangle\right] + \frac{B(2B_{\phi_2}|\phi_1(0)| + |k|)}{\sqrt{N_r}}.
    \end{align}
\end{lemma}



The proofs of the above lemmas are provided in Section~\ref{sec:NSlemproof}.

Building on the preceding lemmas, we proceed to derive the upper bound established in Theorem~\ref{thm:NSRad}.

\subsection{Proof of Theorem \ref{thm:NSRad}}

\begin{proof}
    First we can bound the generalization error by the Rademacher expression of the empirical loss as given by Lemma \ref{lem:genbound}, giving us
    \begin{align*}
       &\E_{\substack{(\vx_{ri},t_{ri}) \sim \gD_1 \forall i \in [N_r] \\ \vx_{0j} \sim \gD_2 \forall j\in[N_0]}} \Bigg[\sup_{h\in\gH} \Bigg( \frac{1}{N_r} \sum_{i=i}^{N_r} \ell_{res}(h,(\vx_{ri},t_{ri})) + \frac{1}{N_0} \sum_{j=i}^{N_0} \ell_{0}(h,\vx_{0j}) \\
        &\hspace{14.00em}- \left( \E_{(\vx_r,t_r) \sim \gD_1}[\ell_{res}(h,(\vx_{r},t_{r}))] + \E_{\vx_0 \sim \gD_2}[\ell_{0}(h,\vx_{0})] \right) \Bigg)\Bigg]\\
        &\leq 2\gR_{res} (\gH) + 2\gR_{0} (\gH)
    \end{align*}

    where,
    \begin{align*}
        \gR_{res} (\gH) &\coloneqq \E_{\substack{(\vx_{ri},t_{ri}) \sim D_1 \forall i \in [N_r]}}\left[ \E_{\epsilon_i\sim \{\pm 1\}} \left[ \sup_{h\in\gH} \left(\frac{1}{N_r} \sum_{i=i}^{N_r} \epsilon_i\ell_{res}(h,(\vx_{ri},t_{ri})) \right) \right] \right] \\
        \gR_0 (\gH) &\coloneqq \E_{\substack{\vx_{0j}' \sim \gD_2 \forall j\in[N_0]}} \left[\E_{\vepsilon \sim \{\pm 1\}^{N_0}} \left[ \sup_{h\in\gH} \left( \frac{1}{N_0} \sum_{j=i}^{N_0} \epsilon_j \ell_{0}(h,\vx_{0j}) \right)\right]\right]
    \end{align*}

    By Definition \ref{def:lossclass}, we substitute the definition of $\ell_{res}$ and $\ell_0$ and the partial derivatives given by Lemma \ref{lem:partialderivatives} to get,
    \begin{align*}
        \gR_{res} &= \E_{\substack{(\vz_i\sim D_1)\\\forall i\in[N_r]}}\Biggl[\E_{\vepsilon\sim\{\pm1\}^{N_r}}\Biggl[\sup_{\mW\in\gC}\biggl(\frac{1}{N_r}\sum_{i=1}^{N_r}\epsilon_i\Bigl\{\sum_{k=1}^d\ell\Bigl(\ip{\va_{1k}\odot\vw_{t}}{\sigma'(\mW\vz)} +\sum_{m=1}^d(\ip{\va_{1m}}{\sigma(\mW\vz_i)})(\ip{\va_{1k}\odot\vw_{\vx_m}}{\sigma'(\mW\vz_i)}) \nonumber\\
        &\hspace{7.00em}+ \ip{\va_{2}\odot\vw_{x_k}}{\sigma'(\mW\vz_i)}-\nu\sum_{m=1}^d\ip{\va_{1k}\odot\vw_{x_m}\odot\vw_{x_m}}{\sigma''(\mW\vz_i)}\Bigr) + \lambda_0\ell\Bigl(\sum_{m=1}^d\ip{\va_{1m}\odot\vw_{x_m}}{\sigma'(\mW\vz_i)}\Bigr)\Bigr\}\biggr)\Biggr]\Biggr],
    \end{align*}
    where $\vz_i = (\vx_{ri}, t_{ri})$.
    \begin{align*}
        \gR_0 = \lambda_1 \E_{\substack{\vx_{0j}' \sim \gD_2 \forall j\in[N_0]}} \left[\E_{\vepsilon \sim \{\pm 1\}^{N_0}} \left[ \sup_{\mW\in\gC} \left( \frac{1}{N_0} \sum_{j=i}^{N_0} \sum_{k=1}^d\epsilon_j \ell \left ( (\gN_{\vw,u}(\vx_{0j},0))_k - (\vf_0 (\vx_{0j}))_k \right ) \right)\right] \right],
    \end{align*}
    
    Firstly we focus on $\gR_{res}$ and we rewrite as the following sum splitting across the regularization for the divergenceless condition.
    
    \begin{align*}
        \gR_{res} &= \E_{\substack{(\vz_i\sim D_1)\\\forall i\in[N_r]}}\Biggl[\E_{\vepsilon\sim\{\pm1\}^{N_r}}\Biggl[\sup_{\mW\in\gC}\biggl(\frac{1}{N_r}\sum_{i=1}^{N_r}\epsilon_i\Bigl\{\sum_{k=1}^d\ell\Bigl(\ip{\va_{1k}\odot\vw_{t}}{\sigma'(\mW\vz)} +\sum_{m=1}^d(\ip{\va_{1m}}{\sigma(\mW\vz_i)})(\ip{\va_{1k}\odot\vw_{\vx_m}}{\sigma'(\mW\vz_i)}) \nonumber\\
        &\hspace{7.00em}+ \ip{\va_{2}\odot\vw_{x_k}}{\sigma'(\mW\vz_i)}-\nu\sum_{m=1}^d\ip{\va_{1k}\odot\vw_{x_m}\odot\vw_{x_m}}{\sigma''(\mW\vz_i)}\Bigr)\Bigr\}\biggr)\Biggr]\Biggr] \\
        &\hspace{2.00em} + \E_{\substack{(\vz_i\sim D_1)\\\forall i\in[N_r]}}\Biggl[\E_{\vepsilon\sim\{\pm1\}^{N_r}}\Biggl[\sup_{\mW\in\gC}\biggl(\frac{1}{N_r}\sum_{i=1}^{N_r}\epsilon_i\Bigl\{\lambda_0\ell\Bigl(\sum_{m=1}^d\ip{\va_{1m}\odot\vw_{x_m}}{\sigma'(\mW\vz_i)}\Bigr)\Bigr\}\biggr)\Biggr]\Biggr]
    \end{align*}
    Using the assumption from Definition \ref{def:nn} that the loss is $L_\ell$-Lipschitz, we can use Talagrand's contraction lemma on each of the terms separately to get,
    \begin{align}\label{eq:Rres}
        \gR_{res}&\leq L_\ell \E_{\substack{(\vz_i\sim D_1)\\\forall i\in[N_r]}}\Biggl[\E_{\vepsilon\sim\{\pm1\}^{N_r}}\Biggl[\sup_{\mW\in\gC}\biggl(\frac{1}{N_r}\sum_{i=1}^{N_r}\epsilon_i\Bigl\{\sum_{k=1}^d\Bigl(\ip{\va_{1k}\odot\vw_{t}}{\sigma'(\mW\vz_i)} \nonumber\\
        &\hspace{14.00em}+\sum_{m=1}^d(\ip{\va_{1m}}{\sigma(\mW\vz_i)})(\ip{\va_{1k}\odot\vw_{\vx_m}}{\sigma'(\mW\vz_i)})\nonumber\\
        &\hspace{14.00em}+\ip{\va_{2}\odot\vw_{x_k}}{\sigma'(\mW\vz_i)}-\nu\sum_{m=1}^d\ip{\va_{1k}\odot\vw_{x_m}\odot\vw_{x_m}}{\sigma''(\mW\vz_i)}\Bigr)\Bigr\}\biggr)\Biggr]\Biggr]\nonumber\\
        &+ L_\ell \E_{\substack{(\vz_i\sim D_1)\\\forall i\in[N_r]}}\Biggl[\E_{\vepsilon\sim\{\pm1\}^{N_r}}\Biggl[\sup_{\mW\in\gC}\biggl(\frac{1}{N_r}\sum_{i=1}^{N_r}\epsilon_i\Bigl\{\lambda_0\sum_{m=1}^d\ip{\va_{1m}\odot\vw_{x_m}}{\sigma'(\mW\vz_i)}\Bigr\}\biggr)\Biggr]\Biggr]
    \end{align}


 Observing the structure of the above we define the following five quantities.
 
    \subsection*{Divide and Conquer}
    Let  $\gR_1,\gR_2,\gR_3,\gR_4,\gR_5$ be such that
    \begin{align*}
        \gR_1&=\E_{\vepsilon\sim\{\pm1\}^{N_r}}\Biggl[\sup_{\mW\in\gC}\biggl(\frac{1}{N_r}\sum_{i=1}^{N_r}\epsilon_i\ip{\va_{1}\odot\vw_{t}}{\sigma'(\mW\vz_i)}\biggr)\Biggr]=\E_{\vepsilon\sim\{\pm1\}^{N_r}}\Biggl[\sup_{\mW\in\gC}\biggl(\frac{1}{N_r}\sum_{i=1}^{N_r}\epsilon_i\ip{f_1(\mW)}{\sigma'(\mW\vz_i)}\biggr)\Biggr]\\
        \gR_2&=\E_{\vepsilon\sim\{\pm1\}^{N_r}}\Biggl[\sup_{\mW\in\gC}\biggl(\frac{1}{N_r}\sum_{i=1}^{N_r}\sum_{k=1}^d\sum_{m=1}^d\epsilon_i\ip{\va_{1m}}{\sigma(\mW\vz_i)}\ip{\va_{1k}\odot\vw_{\vx_m}}{\sigma'(\mW\vz_i)}\biggr)\Biggr]\\
        &=\E_{\vepsilon\sim\{\pm1\}^{N_r}}\Biggl[\sup_{\mW\in\gC}\biggl(\frac{1}{N_r}\sum_{i=1}^{N_r}\sum_{m=1}^d\epsilon_i\ip{f_{2m}(\mW)}{\sigma'(\mW\vz_i)}\ip{\va_{1m}}{\sigma(\mW\vz_i)}\biggr)\Biggr]&\\
        \gR_3&=\E_{\vepsilon\sim\{\pm1\}^{N_r}}\Biggl[\sup_{\mW\in\gC}\biggl(\frac{1}{N_r}\sum_{i=1}^{N_r}\epsilon_i\ip{\va_{2}\odot\vw_\vx}{\sigma'(\mW\vz_i)}\biggr)\Biggr]=\E_{\vepsilon\sim\{\pm1\}^{N_r}}\Biggl[\sup_{\mW\in\gC}\biggl(\frac{1}{N_r}\sum_{i=1}^{N_r}\epsilon_i\ip{f_3(\mW)}{\sigma'(\mW\vz_i)}\biggr)\Biggr]\\
        \gR_4
        &=\E_{\vepsilon\sim\{\pm1\}^{N_r}}\Biggl[\sup_{\mW\in\gC}\biggl(\frac{-\nu}{N_r}\sum_{i=1}^{N_r}\sum_{m=1}^d\epsilon_i\ip{\va_{1}\odot\vw_{x_m}\odot\vw_{x_m}}{\sigma''(\mW\vz_i)}\biggr)\Biggr]=\E_{\vepsilon\sim\{\pm1\}^{N_r}}\Biggl[\sup_{\mW\in\gC}\biggl(\frac{-\nu}{N_r}\sum_{i=1}^{N_r}\epsilon_i\ip{f_4(\mW)}{\sigma''(\mW\vz_i)}\biggr)\Biggr]\\
        \gR_5
        &=\E_{\vepsilon\sim\{\pm1\}^{N_r}}\Biggl[\sup_{\mW\in\gC}\biggl(\frac{\lambda_0}{N_r}\sum_{i=1}^{N_r}\sum_{m=1}^d\epsilon_i\ip{\va_{1m}\odot\vw_{x_m}}{\sigma'(\mW\vz_i)}\biggr)\Biggr]=\E_{\vepsilon\sim\{\pm1\}^{N_r}}\Biggl[\sup_{\mW\in\gC}\biggl(\frac{\lambda_0}{N_r}\sum_{i=1}^{N_r}\epsilon_i\ip{f_5(\mW)}{\sigma'(\mW\vz_i)}\biggr)\Biggr]
    \end{align*}     
        
Hence we can rewrite equation \ref{eq:Rres} to get,

    \begin{align*} 
        \gR_{res}&\leq L_\ell\E_{\substack{(\vx_{ri},t_{ri}\sim D_1)\\\forall i\in[N_r]}}\Biggl[\gR_1+\gR_2+\gR_3+\gR_4+\gR_5\Biggr]&
    \end{align*}


    \textbf{Bounding Terms 1, 3 and 5}

    We recall the Lipschitz constants $L_{\sigma'}$ for $\sigma'(\cdot)$ and $\sigma'(0) = c_1$. Therefore, for $\gR_1, \gR_3$, and $\gR_5$, we can invoke Lemma \ref{lem:dcontract1tanh}, with $\phi(\cdot) = \sigma'(\cdot)$, and recalling the definition of $B_{f_1}, B_{f_3}$ and $B_{f_5}$ from Definition~\ref{def:Wspace} to get
    \begin{align*}
        \gR_{1} &\leq \frac{2 B_{f_1} L_{\sigma'}}{N_r} \E_{\substack{\vepsilon \sim \{\pm 1\}^{N_r}}} \left[ \sup_{\bar{\vw}\in\textrm{RS}(\gC)} \sum_{i=1}^{N_r} \epsilon_i \ip{\bar{\vw}}{\vz_i} \right ] + \frac{B_{f_1} \abs{c_1}}{\sqrt{N_r}} \\
        \gR_{3} &\leq \frac{2 B_{f_3} L_{\sigma'}}{N_r} \E_{\substack{\vepsilon \sim \{\pm 1\}^{N_r}}} \left[ \sup_{\bar{\vw}\in\textrm{RS}(\gC)} \sum_{i=1}^{N_r} \epsilon_i \ip{\bar{\vw}}{\vz_i} \right ] + \frac{B_{f_3} \abs{c_1}}{\sqrt{N_r}}\\
        \gR_{5} &\leq \frac{2 \lambda_0 B_{f_5} L_{\sigma'}}{N_r} \E_{\substack{\vepsilon \sim \{\pm 1\}^{N_r}}} \left[ \sup_{\bar{\vw}\in\textrm{RS}(\gC)} \sum_{i=1}^{N_r} \epsilon_i \ip{\bar{\vw}}{\vz_i} \right] + \frac{B_{f_5} \abs{c_1}}{\sqrt{N_r}}
    \end{align*}

    
    \textbf{Bounding Term 2}

    Since $\sigma(\cdot)$ is $L_{\sigma}$-Lipschitz, $\sigma(\cdot) \leq B_{\sigma}$, $\sigma'(\cdot)$ is $L_{\sigma'}$-Lipschitz, $\sigma'(\cdot) \leq B_{\sigma'}$, and $\sigma(0) \cdot \sigma'(0) = c_0 c_1$ , we can apply Lemma \ref{lem:dcontract2} with, $\phi_1(\cdot) = \sigma(\cdot)$, $\phi_2(\cdot) = \sigma'(\cdot)$ and $k = c_0 c_1$ , and recalling the definition of $B_{f_2}$ from Definition~\ref{def:Wspace} to get
    \begin{align*}
        \gR_2 \leq \frac{4B_{f_2}(B_{\sigma} L_{\sigma'} + B_{\sigma'} L_{\sigma})}{N_r}\, \mathbb{E}_{\epsilon}\left[\sup_{\bar{w} \in RS(C)} \sum_{i=1}^{N_r} \epsilon_i \langle \bar{\vw}, \vz_i\rangle\right] + \frac{B_{f_2}(2B_{\sigma'}|c_0| + |c_0 c_1|)}{\sqrt{N_r}}
    \end{align*}

    
    \textbf{Bounding Term 4}
    
    For the 4th term, since we are taking the expectation for $\epsilon_i \sim \{\pm 1\}$, we can simplify $\gR_4$ as
    \begin{align*}
        \gR_4 =\E_{\vepsilon\sim\{\pm1\}^{N_r}}\Biggl[\sup_{\mW\in\gC}\biggl(\frac{\nu}{N_r}\sum_{i=1}^{N_r}\epsilon_i\ip{f_4(\mW)}{\sigma''(\mW\vz_i)}\biggr)\Biggr].
    \end{align*}
    We recall the Lipschitz constants $L_{\sigma''}$ for $\sigma''(\cdot)$ and $\sigma''(0) = c_2$. Therefore, for $\gR_4$, we can invoke Lemma \ref{lem:dcontract1tanh} with $\phi(\cdot) = \sigma''(\cdot)$ and recalling the definition of $B_{f_4}$ from Definition~\ref{def:Wspace} to get 

    \begin{align*}
        \gR_{4} &\leq \frac{2\nu B_{f_4} L_{\sigma''}}{N_r} \E_{\substack{\vepsilon \sim \{\pm 1\}^{N_r}}} \left[ \sup_{\bar{\vw} \in \textrm{RS}(\gC)} \sum_{i=1}^{N_r} \epsilon_i \ip{\bar{\vw}}{\vz_i} \right ] + \frac{\nu B_{f_4} \abs{c_2}}{\sqrt{N_r}}.
    \end{align*}
    
    \subsection*{Combining the Bounds on Terms 1-5}
    \begin{align*}
        \gR_{res} &\leq L_\ell ( \gR_1 + \gR_2 + \gR_3 + \gR_4 + \gR_5 )\\
        &\leq L_\ell\left(\frac{2B_{f_1}L_{\sigma'}+4B_{f_2}(B_\sigma L_{\sigma'}+B_{\sigma'}L_\sigma)+2B_{f_3}L_{\sigma'}+2\nu B_{f_4}(L_{\sigma'}+L_{\sigma})+2 \lambda_0 B_{f_5}L_{\sigma'}}{N_r}\right)\times \E_{\vepsilon\sim\{\pm1\}^{N_r}}\left[\sup_{\bar{\vw}\in\textrm{RS}(\gC)}\sum_{i=1}^{N_r}\epsilon_i\ip{\bar{\vw}}{\vz_i}\right] \\
        &\hspace{2.00em}+ \frac{L_\ell (B_{f_1}\abs{c_1} + B_{f_2}(2B_{\sigma'}\abs{c_0} + \abs{c_0 c_1}) + B_{f_3}\abs{c_1} + \nu B_{f_4}c_2 + \lambda_0 B_{f_5}\abs{c_1})}{\sqrt{N_r}}
    \end{align*}

        

    So, we are finally left with a Rademacher complexity of a linear model
    \begin{align*}
        \hat{\gR}_{res} \leq \frac{L_\ell C_1}{N_r} \E_{\substack{\vepsilon \sim \{\pm 1\}^{N_r}}} \left[ \sup_{\bar{\vw}\in\textrm{RS}(\gC)} \sum_{i=1}^{N_r} \epsilon_i \ip{\bar{\vw}}{\vz_i}\right] + \frac{L_\ell C_2}{\sqrt{N_r}},
    \end{align*}
    where $C_1 \coloneqq 2B_{f_1}L_{\sigma'}+4B_{f_2}(B_\sigma L_{\sigma'}+B_{\sigma'}L_\sigma)+2B_{f_3}L_{\sigma'}+2\nu B_{f_4}(L_{\sigma'}+L_{\sigma})+2 \lambda_0 B_{f_5}L_{\sigma'}$ and $C_2 \coloneqq  B_{f_1}\abs{c_1} + B_{f_2}(2B_{\sigma'}\abs{c_0} + \abs{c_0 c_1}) + B_{f_3}\abs{c_1} + \nu B_{f_4}\abs{c_2} + \lambda_0 B_{f_5}\abs{c_1}$.
    We recall that $\norm{\vw_q}_2 \leq B_{\vw}, \forall q \in [p]$, where $\vw_q = \mW_{q,:}$, and $\E_{\vz_i}[\norm{\vz_i}_2^2] \leq C_{\vz}^2, \forall i \in [N_r]$, where $\vz_i = (\vx_{ri}, t_{ri})$ from Definition~\ref{def:collocation}. Then using the standard Rademacher complexity bound for linear models (Theorem 5.5 of \cite{tengyucs229m}) we get,
    \begin{align*}
        \hat{\gR}_{res} \leq \frac{L_\ell B_{\vw} C_1}{N_r} \sqrt{\sum_{i=1}^{N_r} \norm{\vz_i}_2^2} + \frac{L_\ell C_2}{\sqrt{N_r}}
    \end{align*}
    and taking the expectation over the data gives us,
    \begin{align}
        \gR_{res} \leq \frac{L_\ell (B_{\vw} C_\vz C_1 + C_2)}{\sqrt{N_r}}.
    \end{align}

    \subsection*{Analysing the Initial Condition Enforcing Term}
    
    \begin{align*}
        \gR_0 = \lambda_1 \E_{\vepsilon \sim \{\pm 1\}^{N_0}} \left[ \sup_{\mW\in\gC} \left( \frac{1}{N_0} \sum_{j=i}^{N_0} \sum_{k=1}^d\epsilon_j \ell \left ( (\gN_{\vw,u}(\vx_{0j},0))_k - (\vf_0 (\vx_{0j}))_k \right ) \right)\right]
    \end{align*}
    
    Since the loss $\ell$ is Lipschitz with constant $L_\ell$ we can apply Talagrand's contraction to get
    \begin{align*}
        &\leq \frac{\lambda_1 L_\ell}{N_0} \E_{\vepsilon \sim \{\pm 1\}^{N_0}} \left[ \sup_{\mW\in\gC} \left( \sum_{j=i}^{N_0} \sum_{k=1}^d\epsilon_j ((\gN_{\vw,u}(\vx_{0j},0))_k - (\vf_0 (\vx_{0j}))_k) \right)\right] \\
        &= \frac{\lambda_1 L_\ell}{N_0} \E_{\vepsilon \sim \{\pm 1\}^{N_0}} \left[ \sup_{\mW\in\gC} \left( \sum_{j=i}^{N_0} \sum_{k=1}^d\epsilon_j (\gN_{\vw,u}(\vx_{0j},0))_k \right)\right].
    \end{align*}
    Let's define $\vz_{0j} \coloneqq (\vx_{0j},0)$, then we can simplify the above as
    \begin{align*}
        &= \frac{\lambda_1 L_\ell}{N_0} \E_{\vepsilon \sim \{\pm 1\}^{N_0}} \left[ \sup_{\mW\in\gC} \left( \sum_{j=i}^{N_0} \epsilon_j \ip{\va_1}{\sigma(\mW\vz_{0j})} \right)\right] \\
        &= \frac{\lambda_1 L_\ell}{N_0} \E_{\vepsilon \sim \{\pm 1\}^{N_0}} \left[ \sup_{\vw_q \in\textrm{RS}(\gC)} \left( \sum_{q=1}^p (\va_1)_q \sum_{j=i}^{N_0} \epsilon_j \sigma(\ip{\vw_q}{\vz_{0j}})\right)\right]\\
        &\leq \frac{\lambda_1 L_\ell}{N_0} \E_{\vepsilon \sim \{\pm 1\}^{N_0}} \left[ \sup_{\vw_r \in\textrm{RS}(\gC)} \left( \sum_{q=1}^p \abs{(\va_1)_q} \max_{r \in [p]} \abs{\sum_{j=i}^{N_0} \epsilon_j \sigma(\ip{\vw_r}{\vz_{0j}})} \right)\right]
    \end{align*}
    Using our assumption that $\sum_{q=1}^p \abs{(\va_1)_q} \leq B_\va$, we can say that,
    \begin{align*}
        &\leq \frac{\lambda_1 L_\ell B_\va}{N_0} \E_{\vepsilon \sim \{\pm 1\}^{N_0}} \left[ \sup_{\vw_r \in\textrm{RS}(\gC)} \left( \max_{r \in [p]} \abs{\sum_{j=i}^{N_0} \epsilon_j \sigma(\ip{\vw_{r}}{\vz_{0j}})}\right)\right] \\
        &\leq \frac{\lambda_1 L_\ell B_\va}{N_0} \E_{\vepsilon \sim \{\pm 1\}^{N_0}} \left[ \sup_{\vw_r \in\textrm{RS}(\gC)} \left( \abs{\sum_{j=i}^{N_0} \epsilon_j \sigma(\ip{\vw_r}{\vz_{0j}})} \right)\right] 
    \end{align*}

    Utilizing the fact that $\sigma(0) = c_0$ if we set $\chi(\cdot) = \phi(\cdot) - c_0$, we note that for any $\vepsilon \in \{\pm 1\}^{N_r}$, $\sup_{\vw} \ip{\vepsilon}{\chi(\mZ^\top\vw)} \geq 0$, where $\mZ = [\vz_1,\dots,\vz_{N_r}]$. Hence, we can apply Lemma~\ref{lem:dabsvalremtanh} to get,
    \begin{align*}
        &\leq \frac{ \lambda_1 L_\ell B_\va}{N_0} \Bigg\{2\E_{\vepsilon \sim \{\pm 1\}^{N_0}} \left[ \sup_{\vw_r\in\textrm{RS}(\gC)} \left( \sum_{j=i}^{N_0} \epsilon_j \sigma(\ip{\vw_r}{\vz_{0j}}) \right)\right] + \abs{c_0} \sqrt{N_0} \Bigg\}
    \end{align*}
    Now, since $\sigma(\cdot)$ is $L_\sigma$-Lipschitz, we can further apply Talagrand's contraction to get,
    \begin{align*}
        &\leq \frac{2 \lambda_1 L_\ell L_\sigma B_\va}{N_0} \E_{\vepsilon \sim \{\pm 1\}^{N_0}} \left[ \sup_{\vw_r\in\textrm{RS}(\gC)} \left( \sum_{j=i}^{N_0} \epsilon_j \ip{\vw_r}{\vz_{0j}} \right)\right] + \frac{\lambda_1 L_\ell B_\va \abs{c_0}}{\sqrt{N_0}}\\
    \end{align*}
    We recall that $\norm{\vw_r}_2 \leq B_{\vw}, \forall r \in [p]$, where $\vw_r = \mW_{r,:}$, and $\E_{\vz_{0j}}[\norm{\vz_{0j}}_2^2] \leq C_{\vz_0}^2, \forall j \in [N_0]$, from Definition~\ref{def:collocation}. Then using the standard Rademacher complexity bound for linear models (Theorem 5.5 of \cite{tengyucs229m}),
    \begin{align*}
        \hat{\gR}_{0} \leq \frac{2 \lambda_1 L_\ell B_{\vw} L_\sigma B_\va}{N_0}\sqrt{\sum_{j=1}^{N_0} \norm{\vz_{0j}}^2} + \frac{\lambda_1 L_\ell B_\va \abs{c_0}}{\sqrt{N_0}}
    \end{align*}
    and taking the expectation over the data gives us,
    \begin{align}
        \gR_{0} \leq \frac{\lambda_1 L_\ell B_\va (2B_{\vw} C_{\vz_0} L_\sigma + \abs{c_0})}{\sqrt{N_0}} . \label{eq:radboundini:1}
    \end{align}
        
    Therefore, combining the above with the bound on $\cR_{\rm res}$ derived earlier we can bound the empirical Rademacher complexity as,
    \begin{align*}
        &\cR_{\rm res} +\cR_0 \leq \frac{L_\ell (B_{\vw} C_\vz C_1 + C_2)}{\sqrt{N_r}} + \frac{\lambda_1 L_\ell B_\va (2B_{\vw} C_{\vz_0} L_\sigma + \abs{c_0})}{\sqrt{N_0}}
    \end{align*}
        where $C_1 \coloneqq 2B_{f_1}L_{\sigma'}+4B_{f_2}(B_\sigma L_{\sigma'}+B_{\sigma'}L_\sigma)+2B_{f_3}L_{\sigma'}+2\nu B_{f_4} L_{\sigma''}+2 \lambda_0 B_{f_5}L_{\sigma'}$ and $C_2 \coloneqq  B_{f_1}\abs{c_1} + B_{f_2}(2B_{\sigma'}\abs{c_0} + \abs{c_0 c_1}) + B_{f_3}\abs{c_1} + \nu B_{f_4}\abs{c_2} + \lambda_0 B_{f_5}\abs{c_1}$.

    
    Further, by setting the loss function $\ell$ to be the Huber loss $\ell_{H,\delta}$ for some $\delta > 0$ as in Definition~\ref{def:huber}, we can simplify the above bound to
    \begin{align}
        &\cR_{\rm res} +\cR_0 \leq \frac{\delta (B_{\vw} C_\vz C_1 + C_2)}{\sqrt{N_r}} + \frac{\lambda_1 \delta B_\va (2 B_{\vw} C_{\vz_0} L_\sigma + \abs{c_0})}{\sqrt{N_0}}
    \end{align}
    where $C_1 \coloneqq 2B_{f_1}L_{\sigma'}+4B_{f_2}(B_\sigma L_{\sigma'}+B_{\sigma'}L_\sigma)+2B_{f_3}L_{\sigma'}+2\nu B_{f_4} L_{\sigma''}+2 \lambda_0 B_{f_5}L_{\sigma'}$ and $C_2 \coloneqq  B_{f_1}\abs{c_1} + B_{f_2}(2B_{\sigma'}\abs{c_0} + \abs{c_0 c_1}) + B_{f_3}\abs{c_1} + \nu B_{f_4}\abs{c_2} + \lambda_0 B_{f_5}\abs{c_1}$.

    
\end{proof}





\section{Lemmas Towards the Generalization Bound for the Navier-Stokes PDE} \label{sec:NSlemproof}

\subsection{Proof of Lemma \ref{lem:genbound}} \label{prf:lem:genbound}
\begin{proof}
    \begin{align*}
        &\sup_{h\in\gH} \left( \frac{1}{N_r} \sum_{i=i}^{N_r} \ell_{res}(h,(\vx_{ri},t_{ri})) + \frac{1}{N_0} \sum_{j=i}^{N_0} \ell_{0}(h,\vx_{0j}) - \left( \E_{(\vx_r,\vt_r) \sim \gD_1}[\ell_{res}(h,(\vx_{r},\vt_{r}))] + \E_{\vx_0 \sim \gD_2}[\ell_{0}(h,\vx_{0})] \right) \right)\\
        &= \sup_{h\in\gH} \left( \frac{1}{N_r} \sum_{i=i}^{N_r} \ell_{res}(h,(\vx_{ri},t_{ri})) + 
        \frac{1}{N_0} \sum_{j=i}^{N_0} \ell_{0}(h,\vx_{0j}) - \left( \E_{(\vx_{ri}',t_{ti}') \sim \gD_1 \forall i \in [N_r]}\left[\frac{1}{N_r} \sum_{i=i}^{N_r} \ell_{res}(h,(\vx_{ri}',t_{ri}'))\right] \right.\right. \\
        &\hspace{26.0em} + \left.\left. \E_{\vx_{0j}' \sim \gD_2 \forall j\in[N_0]}\left[\frac{1}{N_0} \sum_{j=i}^{N_0} \ell_{0}(h,\vx_{0j}')\right] \right) \right)\\
        &= \sup_{h\in\gH} \left( \E_{(\vx_{ri}',t_{ti}') \sim \gD_1 \forall i \in [N_r]} \left[\frac{1}{N_r} \sum_{i=i}^{N_r} \ell_{res}(h,(\vx_{ri},t_{ri})) - \frac{1}{N_r} \sum_{i=i}^{N_r} \ell_{res}(h,(\vx_{ri}',t_{ri}')) \right] \right. \\
        &\hspace{17.0em} \left. + \E_{\vx_{0j}' \sim \gD_2 \forall j\in[N_0]} \left[ \frac{1}{N_0} \sum_{j=i}^{N_0} \ell_{0}(h,\vx_{0j}) - \frac{1}{N_0} \sum_{j=i}^{N_0} \ell_{0}(h,\vx_{0j}')\right] \right)\\
        &\leq \sup_{h\in\gH} \left( \E_{(\vx_{ri}',t_{ti}') \sim \gD_1 \forall i \in [N_r]} \left[ \frac{1}{N_r} \sum_{i=i}^{N_r} \ell_{res}(h,(\vx_{ri},t_{ri})) - \frac{1}{N_r} \sum_{i=i}^{N_r} \ell_{res}(h,(\vx_{ri}',t_{ri}')) \right] \right) \\
        &\hspace{17.0em} + \sup_{h\in\gH} \left( \E_{\vx_{0j}' \sim \gD_2 \forall j\in[N_0]} \left[\frac{1}{N_0} \sum_{j=i}^{N_0} \ell_{0}(h,\vx_{0j}) - \frac{1}{N_0} \sum_{j=i}^{N_0} \ell_{0}(h,\vx_{0j}')\right] \right)\\
        &\leq \E_{(\vx_{ri}',t_{ti}') \sim \gD_1 \forall i \in [N_r]} \left[ \sup_{h\in\gH} \left(\frac{1}{N_r} \sum_{i=i}^{N_r} \ell_{res}(h,(\vx_{ri},t_{ri})) - \frac{1}{N_r} \sum_{i=i}^{N_r} \ell_{res}(h,(\vx_{ri}',t_{ri}'))\right) \right] \\
        &\hspace{17.0em} + \E_{\vx_{0j}' \sim \gD_2 \forall j\in[N_0]} \left[ \sup_{h\in\gH} \left( \frac{1}{N_0} \sum_{j=i}^{N_0} \ell_{0}(h,\vx_{0j}) - \frac{1}{N_0} \sum_{j=i}^{N_0} \ell_{0}(h,\vx_{0j}') \right)\right]
    \end{align*}
    Now taking the expectation over $(\vx_{ri}, t_{ri})$ and $x_{0i}$ on both sides we get
    \begin{align*}
        &\E_{\substack{(\vx_{ri},t_{ri}) \sim \gD_1 \forall i \in [N_r] \\ \vx_{0j} \sim \gD_2 \forall j\in[N_0]}} \left[\sup_{h\in\gH} \left( \frac{1}{N_r} \sum_{i=i}^{N_r} \ell_{res}(h,(\vx_{ri},t_{ri})) + \frac{1}{N_0} \sum_{j=i}^{N_0} \ell_{0}(h,\vx_{0j}) - \left( \E_{(x_r,t_r) \sim \gD_1}[\ell_{res}(h,(x_{r},t_{r}))] + \E_{\vx_0 \sim \gD_2}[\ell_{0}(h,x_{0})] \right) \right)\right]\\
        &\leq \E_{(\vx_{ri},t_{ri}) \sim D_1 \forall i \in [N_r]}\left[ \E_{(\vx_{ri}',t_{ti}') \sim \gD_1 \forall i \in [N_r]} \left[ \sup_{h\in\gH} \left(\frac{1}{N_r} \sum_{i=i}^{N_r} (\ell_{res}(h,(\vx_{ri},t_{ri})) - \ell_{res}(h,(\vx_{ri}',t_{ri}')))\right) \right] \right]\\
        &\hspace{15.0em} + \E_{\vx_{0j} \sim \gD_2 \forall j\in[N_0]} \left[\E_{\vx_{0j}' \sim \gD_2 \forall j\in[N_0]} \left[ \sup_{h\in\gH} \left( \frac{1}{N_0} \sum_{j=i}^{N_0} (\ell_{0}(h,\vx_{0j}) - \ell_{0}(h,\vx_{0j}')) \right)\right]\right]\\
        &= \E_{\substack{(\vx_{ri},t_{ri}) \sim D_1 \forall i \in [N_r] \\ (\vx_{ri}',t_{ti}') \sim \gD_1 \forall i \in [N_r]}}\left[ \E_{\epsilon_i\sim \{\pm 1\}} \left[ \sup_{h\in\gH} \left(\frac{1}{N_r} \sum_{i=i}^{N_r} \epsilon_i(\ell_{res}(h,(\vx_{ri},t_{ri})) - \ell_{res}(h,(\vx_{ri}',t_{ri}')))\right) \right] \right]\\
        &\hspace{15.0em} + \E_{\substack{\vx_{0j} \sim \gD_2 \forall j\in[N_0] \\ \vx_{0j}' \sim \gD_2 \forall j\in[N_0]}} \left[\E_{\epsilon_j \sim \{\pm 1\}} \left[ \sup_{h\in\gH} \left( \frac{1}{N_0} \sum_{j=i}^{N_0} \epsilon_j (\ell_{0}(h,\vx_{0j}) - \ell_{0}(h,\vx_{0j}')) \right)\right]\right]\\
        &\leq \E_{\substack{(\vx_{ri},t_{ri}) \sim D_1 \forall i \in [N_r] \\ (\vx_{ri}',t_{ti}') \sim \gD_1 \forall i \in [N_r]}}\left[ \E_{\epsilon_i\sim \{\pm 1\}} \left[ \sup_{h\in\gH} \left(\frac{1}{N_r} \sum_{i=i}^{N_r} \epsilon_i\ell_{res}(h,(\vx_{ri},t_{ri})) \right) + \sup_{h\in\gH} \left( \frac{1}{N_r} \sum_{i=i}^{N_r} - \epsilon_i\ell_{res}(h,(\vx_{ri}',t_{ri}')) \right) \right] \right]\\
        &\hspace{15.0em} + \E_{\substack{\vx_{0j} \sim \gD_2 \forall j\in[N_0] \\ \vx_{0j}' \sim \gD_2 \forall j\in[N_0]}} \left[\E_{\epsilon_j \sim \{\pm 1\}} \left[ \sup_{h\in\gH} \left( \frac{1}{N_0} \sum_{j=i}^{N_0} \epsilon_j \ell_{0}(h,\vx_{0j}) \right) + \sup_{h\in\gH} \left( \frac{1}{N_0} \sum_{j=i}^{N_0} - \epsilon_j \ell_{0}(h,\vx_{0j}') \right)\right]\right]\\
        &= 2\gR_{res} (\gH) + 2\gR_{0} (\gH)
    \end{align*}
\end{proof}

\subsection{Proof of Lemma \ref{lem:partialderivatives}}\label{prf:lem:partialderivatives}

\begin{proof}
The net's full output is given by:
\[(\gN_{\vw,u}(\vz))_i=\sum_{j=1}^pa_{1ij}\sigma\left(\sum_{k=1}^dw_{jk}x_k+w_{j,d+1}t\right),\]
where $\vz = (\vx, t)$. The first order partial derivatives of $\gN_{\vw}$ are 
\begin{align*}
    \partial_{x_k} \gN_{\vw,p}(\vz) &= \sum_{j=1}^p a_{2j} \sigma'\left(\sum_{q=1}^dw_{j,q}x_q+w_{j,d+1}t\right)w_{j,k}=\ip{\va_2\odot\vw_{x_k}}{\sigma'(\mW\vz)}\\
    (\partial_t\gN_{\vw,\vu}(\vz))_k&=\sum_{j=1}^pa_{1k,j}\sigma'\left(\sum_{q=1}^dw_{j,q}x_q+w_{j,d+1}t\right)w_{j,d+1}=\ip{\va_{1k}\odot\vw_t}{\sigma'(\mW\vz)}\\
    (\partial_{x_m}\gN_{\vw,\vu}(\vz))_k&=\sum_{j=1}^pa_{1k,j}\sigma'\left(\sum_{q=1}^dw_{j,q}x_q+w_{j,d+1}t\right)w_{j,m}=\ip{\va_{1k}\odot\vw_{x_m}}{\sigma'(\mW\vz)}\\
\end{align*}
Using these we can show the non-linear term can be expressed as
\begin{align*}
    \nabla\cdot\gN_{\vw,u}(\vz)&=\sum_{m=1}^d\sum_{j=1}^pa_{1m,j}\sigma'\left(\sum_{q=1}^dw_{j,q}x_q+w_{j,d+1}t\right)w_{j,m}=\sum_{m=1}^d\ip{\va_{1m}\odot\vw_{x_m}}{\sigma'(\mW\vz)}\\
    \implies((\gN_{\vw,\vu}\cdot\nabla)\gN_{\vw,\vu})_k&=\sum_{m=1}^d\left(\sum_{j=1}^pa_{1m,j}\sigma\left(\sum_{q=1}^dw_{j,q}x_q+w_{j,d+1}t\right)\right) \left(\sum_{j=1}^pa_{1k,j}\sigma'\left(\sum_{q=1}^dw_{j,q}x_q+w_{j,d+1}t\right)w_{j,m}\right)\\
    &=\sum_{m=1}^d(\ip{\va_{1m}}{\sigma(\mW\vz)})(\ip{\va_{1k}\odot\vw_{\vx_m}}{\sigma'(\mW\vz)})\\
\end{align*}
and that the Laplacian is
\begin{align*}
    (\partial_{x_m}^2\gN_{\vw,\vu}(\vz))_k&=\sum_{j=1}^pa_{1k,j}\sigma''\left(\sum_{q=1}^dw_{j,q}x_q+w_{j,d+1}t\right)w_{j,m}^2=\ip{\va_{1k}\odot\vw_{x_m}\odot\vw_{x_m}}{\sigma''(\mW\vz)}\\
    \implies(\nabla^2\gN_{\vw,\vu}(\vz))_k&=\sum_{m=1}^d\sum_{j=1}^pa_{1k,j}\sigma''\left(\sum_{q=1}^dw_{j,q}x_q+w_{j,d+1}t\right)w_{j,m}^2=\sum_{m=1}^d\ip{\va_{1k}\odot\vw_{x_m}\odot\vw_{x_m}}{\sigma''(\mW\vz)}
\end{align*}

\end{proof}




\subsection{Proof of Lemma \ref{lem:dabsvalremtanh}}\label{prf:lem:dabsvalremtanh}

\begin{proof}
    \begin{align*}
        \sup_{\vw}\abs{\sum_{i=1}^n \epsilon_i f_\vw(\vz_i)} &= \sup_{\vw}\abs{\sum_{i=1}^n \epsilon_i (g_\vw(\vz_i) + c)}\leq\sup_\vw \left( \abs{\sum_{i=1}^n \epsilon_i g_\vw(\vz_i)} + \abs{\sum_{i=1}^n c\epsilon_i}\right) = \sup_\vw \left( \abs{\sum_{i=1}^n \epsilon_i g_\vw(\vz_i)}\right) + \abs{\sum_{i=1}^n c\epsilon_i}
    \end{align*}
    Now, taking the expectation over $\vepsilon$ on both sides we get,
    \begin{align}
        \E_{\vepsilon \sim \{\pm 1\}^n} \left[ \sup_{\vw}\abs{\sum_{i=1}^n \epsilon_i f_\vw(\vz_i)} \right] &\leq \E_{\vepsilon \sim \{\pm 1\}^n} \left[ \sup_\vw \left( \abs{\sum_{i=1}^n \epsilon_i g_\vw(\vz_i)}\right) \right] + \E_{\vepsilon \sim \{\pm 1\}^n} \left[ \abs{\sum_{i=1}^n c\epsilon_i} \right]\nonumber\\
        &\leq \E_{\vepsilon \sim \{\pm 1\}^n} \left[ \sup_\vw \left( \abs{\sum_{i=1}^n \epsilon_i g_\vw(\vz_i)}\right) \right] + \abs{c}\sqrt{n} \label{eq:minitala:1}
    \end{align}
    Let $\phi$ be the ReLU function, then the lemma's assumption implies that $\sup_{\vw} \phi\left(\ip{\vepsilon}{g_{\vw}(\mZ)}\right)=\sup_{\vw}\ip{\vepsilon}{g_{\vw}(\mZ)}$ for any $\vepsilon \in\{\pm 1\}^{n}$. Observing that $|x|=\phi(x)+\phi(-x)$,
    \begin{align*}
        \sup_{\vw}\abs{\ip{\vepsilon}{g_{\vw}(\mZ)}} &=\sup_{\vw}\left[\phi\left(\ip{\vepsilon}{g_{\vw}(\mZ)}\right)+\phi\left(\ip{-\vepsilon}{g_{\vw}(\mZ)}\right)\right] \\
        & \leq \sup _{\vw} \phi\left(\ip{\vepsilon}{g_{\vw}(\mZ)}\right)+\sup _{\vw} \phi\left(\ip{-\vepsilon}{g_{\vw}(\mZ)}\right) \\
        &=\sup_{\vw}\ip{\vepsilon}{g_{\vw}(\mZ)} + \sup_{\vw}\ip{-\vepsilon}{g_{\vw}(\mZ)}.
    \end{align*}
    Taking the expectation over $\vepsilon$ (and noting that $\vepsilon$ and $-\vepsilon$ have the same distribution) we get,
    \begin{align*}
        \E_{\vepsilon \sim \{\pm 1\}^n} \left[ \sup_\vw \abs{\sum_{i=1}^n \epsilon_i g_\vw(\vz_i)} \right] \leq 2\E_{\vepsilon \sim \{\pm 1\}^n} \left[ \sup_\vw \left(\sum_{i=1}^n \epsilon_i g_\vw(\vz_i)\right) \right]
    \end{align*}
    Replacing the above into equation~\ref{eq:minitala:1} we get,
    \begin{align}
        \E_{\vepsilon \sim \{\pm 1\}^n} \left[ \sup_{\vw}\abs{\sum_{i=1}^n \epsilon_i f_\vw(\vz_i)} \right] &\leq 2\E_{\vepsilon \sim \{\pm 1\}^n} \left[ \sup_\vw \left( \sum_{i=1}^n \epsilon_i g_\vw(\vz_i)\right) \right] + \abs{c}\sqrt{n}
    \end{align}
\end{proof}

\subsection{Proof of Lemma \ref{lem:dcontract1tanh}}\label{prf:lem:dcontract1tanh}

\begin{proof}
    \begin{align*}
        &\E_{\vepsilon \sim \{\pm 1\}^{N_r}} \left[ \sup_{\mW\in\gC} \left( \frac{1}{N_r} \sum_{i=1}^{N_r} \epsilon_i \ip{f(\mW)}{\phi(\mW \vz_i)} \right) \right] \nonumber\\
        &= \frac{1}{N_r} \E_{\vepsilon \sim \{\pm 1\}^{N_r}} \left[ \sup_{\mW\in\gC} \left( \sum_{i=1}^{N_r} \epsilon_i \sum_{m=1}^{p} f(\mW)_m \phi(\ip{\vw_m}{\vz_i}) \right) \right] \\
        &= \frac{1}{N_r} \E_{\vepsilon \sim \{\pm 1\}^{N_r}} \left[ \sup_{\mW\in\gC} \left( \sum_{m=1}^{p} f(\mW)_m \sum_{i=1}^{N_r} \epsilon_i \phi(\ip{\vw_m}{\vz_i}) \right) \right] \nonumber\\
        &\leq \frac{1}{N_r} \E_{\vepsilon \sim \{\pm 1\}^{N_r}} \left[ \sup_{\mW\in\gC} \left( \sum_{m=1}^{p} \abs{f(\mW)_m} \max_{k\in[p]} \abs{\sum_{i=1}^{N_r} \epsilon_i \phi(\ip{\vw_k}{\vz_i})} \right) \right] &&\left( \text{since } \sum_j \alpha_j \beta_j \leq \sum_j \abs{\alpha_j} \max_k \abs{\beta_k} \right)\\
        &\leq \frac{B}{N_r} \E_{\vepsilon \sim \{\pm 1\}^{N_r}} \left[ \sup_{\mW\in\gC} \max_{k\in[p]} \abs{\sum_{i=1}^{N_r} \epsilon_i \phi(\ip{\vw_k}{\vz_i})} \right] &&\left( \text{since } \sum_{m=1}^{p} \abs{f(\mW)_m} \leq B \right)\\
        &= \frac{B}{N_r} \E_{\vepsilon \sim \{\pm 1\}^{N_r}} \left[ \sup_{\bar{\vw}\in\textrm{RS}(\gC)} \abs{\sum_{i=1}^{N_r} \epsilon_i \phi(\ip{\bar{\vw}}{\vz_i})} \right]
    \end{align*}
    We then apply Lemma \ref{lem:dabsvalremtanh} with $f_{\bar{\vw}}(\vz_i) = \phi(\langle \bar{\vw}, \vz_i\rangle)$ and constant $c = \phi_1(0)$, so that $g_{\bar{\vw}}(\vz_i) = \phi(\langle \bar{\vw}, \vz_i\rangle)-c$ vanishes at $\bar{\vw} = 0$. The condition $\sup_{\bar{\vw}} \langle \epsilon, g_{\bar{\vw}}(\mZ)\rangle \geq 0$ holds since the supremum is taken over the zero vector as a feasible point. Lemma \ref{lem:dabsvalremtanh} then gives:
    
    \begin{align*}
        &\leq \frac{B}{N_r} \left\{2\E_{\vepsilon \sim \{\pm 1\}^{N_r}} \left[ \sup_{\bar{\vw}\in\textrm{RS}(\gC)} \sum_{i=1}^{N_r} \epsilon_i g(\ip{\bar{\vw}}{\vz_i}) \right] + \abs{c}\sqrt{N_r}\right\}
    \end{align*}    
    Now, $\chi(\cdot)$ has the same Lipschitz constant as $\phi(\cdot)$,  we can apply Talagrand's contraction lemma to get,
    \begin{align}
        &\leq \frac{2B L_{\phi}}{N_r} \E_{\vepsilon \sim \{\pm 1\}^{N_r}} \left[ \sup_{\bar{\vw}\in\textrm{RS}(\gC)} \sum_{i=1}^{N_r} \epsilon_i \ip{\bar{\vw}}{\vz_i} \right] + \frac{B\abs{c}}{\sqrt{N_r}}
    \end{align}
\end{proof}

\subsection{Proof of Lemma \ref{lem:dcontract2}}\label{prf:lem:dcontract2}

\begin{proof}
    \begin{align*}
        &\E_{\substack{\vepsilon \sim \{\pm 1\}^{N_r}}} \left[ \sup_{\mW\in\gC} \left(\frac{1}{N_r} \sum_{i=1}^{N_r}\sum_{m=1}^d \epsilon_i \ip{f_m(\mW)}{\phi_1 (\mW \vz_i)} \ip{\va_m}{\phi_2(\mW \vz_i)} \right)\right] \nonumber\\
        &=\frac{1}{N_r} \E_{\substack{\vepsilon \sim \{\pm 1\}^{N_r}}} \left[ \sup_{\mW\in\gC} \left( \sum_{i=1}^{N_r}\sum_{m=1}^d \epsilon_i \sum_{q_1,q_2 = 1}^p f_m(\mW)_{q_1} \phi_1(\ip{\vw_{q_1}}{\vz_i}) \va_{m,q_2}\phi_2(\ip{\vw_{q_2}}{\vz_i}) \right)\right] \nonumber\\
        &=\frac{1}{N_r}\E_{\substack{\vepsilon \sim \{\pm 1\}^{N_r}}} \left[ \sup_{\mW\in\gC} \left( \sum_{m=1}^d\sum_{q_1,q_2 = 1}^p f_m(\mW)_{q_1} \va_{m,q_2} \sum_{i=1}^{N_r} \epsilon_i  \phi_1(\ip{\vw_{q_1}}{\vz_i})\phi_2(\ip{\vw_{q_2}}{\vz_i}) \right)\right]
        \end{align*}
     Using $\sum_j \alpha_j \beta_j \leq \sum_j \abs{\alpha_j} \max_k \abs{\beta_k}$,
        \begin{align*}
        &\leq\frac{1}{N_r}\E_{\substack{\vepsilon \sim \{\pm 1\}^{N_r}}} \left[ \sup_{\mW\in\gC} \left( \sum_{m=1}^d\sum_{q_1,q_2 = 1}^p \abs{f_m(\mW)_{q_1}} \abs{\va_{m,q_2}} \max_{q_1,q_2 \in [p]} \abs{\sum_{i=1}^{N_r} \epsilon_i  \phi_1(\ip{\vw_{q_1}}{\vz_i})\phi_2(\ip{\vw_{q_2}}{\vz_i})} \right)\right]
    \end{align*}
    Since $\sum_{m=1}^d\sum_{q_1,q_2 = 1}^p \abs{f_m(\mW)_{q_1}} \abs{\va_{m,q_2}} \leq B$,
    
    
    \begin{align*}
        &\leq\frac{B}{N_r}\E_{\substack{\vepsilon \sim \{\pm 1\}^{N_r}}} \left[ \sup_{\mW\in\gC} \max_{q_1,q_2 \in [p]} \abs{\sum_{i=1}^{N_r} \epsilon_i  \phi_1(\ip{\vw_{q_1}}{\vz_i})\phi_2(\ip{\vw_{q_2}}{\vz_i})} \right]\\
        &= \frac{B}{N_r}\E_{\substack{\vepsilon \sim \{\pm 1\}^{N_r}}} \left[ \sup_{\bar{\vw_1},\bar{\vw_2}\in\textrm{RS}(\gC)} \abs{\sum_{i=1}^{N_r} \epsilon_i  \phi_1(\ip{\bar{\vw}_1}{\vz_i})\phi_2(\ip{\bar{\vw}_2}{\vz_i})} \right]
    \end{align*}
    We then apply Lemma \ref{lem:dabsvalremtanh} with $f_\vw(\vz_i) = \phi_1(\langle \vw_1, \vz_i\rangle)\phi_2(\langle \vw_2, \vz_i\rangle)$ and constant $c = \phi_1(0)\cdot\phi_2(0) = k$, so that $g_\vw(z_i) = \phi_1(\langle \vw_1, \vz_i\rangle)\phi_2(\langle \vw_2, \vz_i\rangle) - k$ vanishes at $\vw = 0$. The condition $\sup_\vw \langle \epsilon, g_\vw(\mZ)\rangle \geq 0$ holds since the supremum is taken over the zero vector as a feasible point. Lemma \ref{lem:dabsvalremtanh} then gives:

    \begin{align}
        &\nonumber\leq\frac{B}{N_r} \left\{2\E_{\substack{\vepsilon \sim \{\pm 1\}^{N_r}}} \left[ \sup_{\bar{\vw}_1\bar{\vw}_2\in\textrm{RS}(\gC)} \sum_{i=1}^{N_r} \epsilon_i  (\phi_1(\ip{\bar{\vw}_{1}}{\vz_i})\phi_2(\ip{\bar{\vw}_2}{\vz_i}) - k)\right] + \abs{k}\sqrt{N_r}\right\}\\
        &\nonumber=\frac{B}{N_r} \left\{2\E_{\substack{\vepsilon \sim \{\pm 1\}^{N_r}}} \left[ \sup_{\bar{\vw}_1\bar{\vw}_2\in\textrm{RS}(\gC)} \sum_{i=1}^{N_r} \epsilon_i  \phi_1(\ip{\bar{\vw}_{1}}{\vz_i})\phi_2(\ip{\bar{\vw}_2}{\vz_i})\right] -2\E_{\substack{\vepsilon \sim \{\pm 1\}^{N_r}}} \left[\sum_{i=1}^{N_r} \epsilon_i k\right]+ \abs{k}\sqrt{N_r}\right\}\\
        &=\frac{B}{N_r} \left\{2\E_{\substack{\vepsilon \sim \{\pm 1\}^{N_r}}} \left[ \sup_{\bar{\vw}_1\bar{\vw}_2\in\textrm{RS}(\gC)} \sum_{i=1}^{N_r} \epsilon_i  \phi_1(\ip{\bar{\vw}_{1}}{\vz_i})\phi_2(\ip{\bar{\vw}_2}{\vz_i})\right] + \abs{k}\sqrt{N_r}\right\} \label{eq:after_lemma5}
    \end{align}
where the last step comes from the fact that $\E_{\substack{\vepsilon \sim \{\pm 1\}^{N_r}}} \left[\sum_{i=1}^{N_r} \epsilon_i \right]=\sum_{i=1}^{N_r}\E_{\substack{\vepsilon \sim \{\pm 1\}^{N_r}}} \left[ \epsilon_i \right]=0$.

It remains to bound the Rademacher complexity of the function class
$\mathcal{G} = \{\psi_{\vw_1,\vw_2} : \vw_1, \vw_2 \in \mathrm{RS}(C)\}$,
where $\psi_{\vw_1,\vw_2}(z) := \phi_1(\langle \vw_1, z\rangle)\phi_2(\langle \vw_2, z\rangle)$.
We proceed in two sequential contraction steps, freezing one supremum 
variable at a time.

\medskip
\noindent\textbf{$\bullet$ (Splitting across a reference weight ${\vw}_2^*$)}
Fix any reference vector $\vw_2^* \in \mathrm{RS}(C)$. For any fixed $\bar{\vw}_1$, 
the scalars $\phi_1(\langle \bar{\vw}_1, \vz_i\rangle)$ are deterministic, 
and since $|\phi_1(\cdot)| \leq B_{\phi_1}$, we have:
\begin{align}
    &\mathbb{E}_{\epsilon}\left[\sup_{\bar{\vw}_1, \bar{\vw}_2 \in \mathrm{RS}(C)} \sum_{i=1}^{N_r} \epsilon_i\, \phi_1(\langle \bar{\vw}_1, \vz_i\rangle)\phi_2(\langle \bar{\vw}_2, \vz_i\rangle)\right] \notag\\
    &\leq \mathbb{E}_{\epsilon}\left[\sup_{\bar{\vw}_1 \in \mathrm{RS}(C)}\left\{ \sup_{\bar{\vw}_2 \in \mathrm{RS}(C)} \sum_{i=1}^{N_r} \epsilon_i\, \phi_1(\langle \bar{\vw}_1, \vz_i\rangle) \left[\phi_2(\langle \bar{\vw}_2, \vz_i\rangle) - \phi_2(\langle \vw_2^*, \vz_i\rangle)\right]  + \sum_{i=1}^{N_r} \epsilon_i\, \phi_1(\langle \bar{\vw}_1, \vz_i\rangle)\phi_2(\langle \vw_2^*, \vz_i\rangle)\right\}\right] \notag\\
    &\leq \mathbb{E}_{\epsilon}\left[\sup_{\bar{\vw}_1, \bar{\vw}_2 \in \mathrm{RS}(C)} \sum_{i=1}^{N_r} \epsilon_i\, \phi_1(\langle \bar{\vw}_1, \vz_i\rangle) \left[\phi_2(\langle \bar{\vw}_2, \vz_i\rangle) - \phi_2(\langle \vw_2^*, \vz_i\rangle)\right]\right] 
     + \mathbb{E}_{\epsilon}\left[\sup_{\bar{\vw}_1 \in \mathrm{RS}(C)} \sum_{i=1}^{N_r} \epsilon_i\, \phi_1(\langle \bar{\vw}_1, \vz_i\rangle)\phi_2(\langle \vw_2^*, \vz_i\rangle)\right]. \label{eq:split}
\end{align}
For the first term in \eqref{eq:split}, since $|\phi_1(\langle \bar{\vw}_1, \vz_i\rangle)| \leq B_{\phi_1}$ we get,
\begin{align*}
    &\mathbb{E}_{\epsilon}\left[\sup_{\bar{\vw}_1, \bar{\vw}_2 \in \mathrm{RS}(C)} \sum_{i=1}^{N_r} \epsilon_i\, \phi_1(\langle \bar{\vw}_1, \vz_i\rangle) \left[\phi_2(\langle \bar{\vw}_2, \vz_i\rangle) - \phi_2(\langle \vw_2^*, \vz_i\rangle)\right]\right]\\ 
    &\leq \mathbb{E}_{\epsilon}\left[\sup_{\bar{\vw}_2 \in \mathrm{RS}(C)} \left|\sum_{i=1}^{N_r} \epsilon_i \phi_1(\bar{\vw}_1, \vz_i) \left[\phi_2(\langle \bar{\vw}_2, \vz_i\rangle) - \phi_2(\langle \vw_2^*, \vz_i\rangle)\right]\right|\right]\\
    &\leq B_{\phi_1}\, \mathbb{E}_{\epsilon}\left[\sup_{\bar{\vw}_2 \in \mathrm{RS}(C)} \left|\sum_{i=1}^{N_r} \epsilon_i \left[\phi_2(\langle \bar{\vw}_2, \vz_i\rangle) - \phi_2(\langle \vw_2^*, \vz_i\rangle)\right]\right|\right]
\end{align*}
We apply Lemma~\ref{lem:dabsvalremtanh} with $f_\vw(\vz_i) = \phi_2(\langle \bar{\vw}_2, \vz_i\rangle) - \phi_2(\langle \vw_2^*, \vz_i\rangle)$ and constant $c = 0$,
so that $g_w(\vz_i) = \phi_2(\langle \bar{\vw}_2, \vz_i\rangle) - \phi_2(\langle \vw_2^*, \vz_i\rangle)$ vanishes at $\bar{\vw}_2 = \vw_2^*$.
The condition $\sup_\vw \langle \epsilon, g_w(\mZ)\rangle \geq 0$ holds since 
the supremum is taken over the zero vector as a feasible point. Therefore, for the first term in the RHS of \eqref{eq:split} we get,
\begin{align*}
    \mathbb{E}_{\epsilon}\left[\sup_{\bar{\vw}_1, \bar{\vw}_2 \in \mathrm{RS}(C)} \sum_{i=1}^{N_r} \epsilon_i\, \phi_1(\langle \bar{\vw}_1, \vz_i\rangle) \left[\phi_2(\langle \bar{\vw}_2, \vz_i\rangle) - \phi_2(\langle \vw_2^*, \vz_i\rangle)\right]\right] &\leq 2 B_{\phi_1}\, \mathbb{E}_{\epsilon}\left[\sup_{\bar{\vw}_2 \in \mathrm{RS}(C)} \sum_{i=1}^{N_r} \epsilon_i \left[\phi_2(\langle \bar{\vw}_2, \vz_i\rangle) - \phi_2(\langle \vw_2^*, \vz_i\rangle)\right]\right]\\
    &\leq 2B_{\phi_1} L_{\phi_2}\, \mathbb{E}_{\epsilon}\left[\sup_{\bar{\vw} \in \mathrm{RS}(C)} \sum_{i=1}^{N_r} \epsilon_i \langle \bar{\vw}, \vz_i\rangle\right],
\end{align*}
where the last step uses Talagrand's contraction on 
$\phi_2(\langle \bar{\vw}_2, \vz_i\rangle) - \phi_2(\langle \vw_2^*, \vz_i\rangle)$,
which is $L_{\phi_2}$-Lipschitz in $\langle \bar{\vw}_2, \vz_i\rangle$.

\medskip
\noindent \textbf{$\bullet$ (Contract over $\bar{\vw}_1$ with $\vw_2^*$ fixed)}
We now handle the second term in \eqref{eq:split}. 
Since $\vw_2^*$ is a fixed reference vector, $\phi_2(\langle \vw_2^*, \vz_i\rangle)$ 
is a fixed scalar for each $i$. Using $|\phi_2(\langle \vw_2^*, \vz_i\rangle)| \leq B_{\phi_2}$:
\begin{align*}
    &\mathbb{E}_{\epsilon}\left[\sup_{\bar{\vw}_1 \in \mathrm{RS}(C)} \sum_{i=1}^{N_r} \epsilon_i\, \phi_1(\langle \bar{\vw}_1, \vz_i\rangle)\phi_2(\langle \vw_2^*, \vz_i\rangle)\right]\leq \mathbb{E}_{\epsilon}\left[\sup_{\bar{\vw}_1 \in \mathrm{RS}(C)} \abs{\sum_{i=1}^{N_r} \epsilon_i\, \phi_1(\langle \bar{\vw}_1, \vz_i\rangle)\phi_2(\langle \vw_2^*, \vz_i\rangle)}\right]\\
    &\leq B_{\phi_2}\, \mathbb{E}_{\epsilon}\left[\sup_{\bar{\vw}_1 \in \mathrm{RS}(C)} \left|\sum_{i=1}^{N_r} \epsilon_i\, \phi_1(\langle \bar{\vw}_1, \vz_i\rangle)\right|\right].
\end{align*}
We apply Lemma~\ref{lem:dabsvalremtanh} 
with $f_\vw(\vz_i) = \phi_1(\langle \vw, \vz_i\rangle)$ and constant $c = \phi_1(0) = c_0$,
so that $g_\vw(\vz_i) = \phi_1(\langle \vw, \vz_i\rangle) - c_0$ vanishes at $\vw = 0$.
The condition $\sup_w \langle \epsilon, g_\vw(Z)\rangle \geq 0$ holds since 
the supremum is taken over the zero vector as a feasible point. Hence for the second term in the RHS of \eqref{eq:split} we get,
%
\begin{align*}
    \mathbb{E}_{\epsilon}\left[\sup_{\bar{\vw}_1 \in \mathrm{RS}(C)} \sum_{i=1}^{N_r} \epsilon_i\, \phi_1(\langle \bar{\vw}_1, \vz_i\rangle)\phi_2(\langle \vw_2^*, \vz_i\rangle)\right] \leq 2 B_{\phi_2}\,\mathbb{E}_{\epsilon}\left[\sup_{\bar{\vw}_1 \in \mathrm{RS}(C)} \sum_{i=1}^{N_r} \epsilon_i \left[\phi_1(\langle \bar{\vw}_1, \vz_i\rangle) - c_0\right]\right] + B_{\phi_2}|c_0|\sqrt{N_r}.
\end{align*}
Now $\phi_1(\langle \bar{\vw}_1, \vz_i\rangle) - c_0$ is an 
$L_{\phi_1}$-Lipschitz function of $\langle \bar{\vw}_1, \vz_i\rangle$, so Talagrand's contraction applies to give,
\begin{align*}
    \mathbb{E}_{\epsilon}\left[\sup_{\bar{\vw}_1 \in \mathrm{RS}(C)} \sum_{i=1}^{N_r} \epsilon_i\, \phi_1(\langle \bar{\vw}_1, \vz_i\rangle)\phi_2(\langle \vw_2^*, \vz_i\rangle)\right] \leq 2 B_{\phi_2} L_{\phi_1}\,\mathbb{E}_{\epsilon}\left[\sup_{\bar{\vw} \in \mathrm{RS}(C)} \sum_{i=1}^{N_r} \epsilon_i \langle \bar{\vw}, \vz_i\rangle\right] + B_{\phi_2}|c_0|\sqrt{N_r}.
\end{align*}

\medskip
\noindent\textbf{$\bullet$ (Combining)} Substituting into \eqref{eq:split} the end results of the above two analysis blocks we get,
\begin{align*}
    \mathbb{E}_{\epsilon}\left[\sup_{\bar{\vw}_1, \bar{\vw}_2 \in \mathrm{RS}(C)} \sum_{i=1}^{N_r} \epsilon_i\, \phi_1(\langle \bar{\vw}_1, \vz_i\rangle)\phi_2(\langle \bar{\vw}_2, \vz_i\rangle)\right]
    \leq 2(B_{\phi_1} L_{\phi_2} + B_{\phi_2} L_{\phi_1})\, \mathbb{E}_{\epsilon}\left[\sup_{\bar{\vw} \in \mathrm{RS}(C)} \sum_{i=1}^{N_r} \epsilon_i \langle \bar{\vw}, \vz_i\rangle\right] + B_{\phi_2}|c_0|\sqrt{N_r}.
\end{align*}
Substituting the above back into \eqref{eq:after_lemma5} we get the desired bound,
\begin{align}
    &\mathbb{E}_{\epsilon \sim \{\pm 1\}^{N_r}}\left[\sup_{W \in C} \left(\frac{1}{N_r}\sum_{i=1}^{N_r} \sum_{m=1}^{d} \epsilon_i \langle f_m(W), \phi_1(W\vz_i)\rangle \langle a_m, \phi_2(W\vz_i)\rangle \right)\right] \notag \\
    &\leq \frac{B}{N_r}\left\{4(B_{\phi_1} L_{\phi_2} + B_{\phi_2} L_{\phi_1})\, \mathbb{E}_{\epsilon}\left[\sup_{\bar{\vw} \in \mathrm{RS}(C)} \sum_{i=1}^{N_r} \epsilon_i \langle \bar{\vw}, \vz_i\rangle\right] + 2 B_{\phi_2}|c_0|\sqrt{N_r} + |k|\sqrt{N_r}\right\}\\
    &=\frac{4B(B_{\phi_1} L_{\phi_2} + B_{\phi_2} L_{\phi_1})}{N_r}\, \mathbb{E}_{\epsilon}\left[\sup_{\bar{\vw} \in \mathrm{RS}(C)} \sum_{i=1}^{N_r} \epsilon_i \langle \bar{\vw}, \vz_i\rangle\right] + \frac{B(2B_{\phi_2}|c_0| + |k|)}{\sqrt{N_r}}. \nonumber
\end{align}

\end{proof}

\section{Experimental Study on the Taylor-Green Vortex}\label{sec:num_experiments}

We choose to test our result on the Taylor-Green vortex solution \citep{taylor1937mechanism} to the Navier-Stokes equations (Equation \ref{eq:dNS}) with space dimension $d=2$.\footnote{Link to the code : \url{https://github.com/SebastienAndreSloan/NS_rademacher}} The true solution we test against is given as,
\begin{align*}\setlength\abovedisplayskip{6pt}\setlength\belowdisplayskip{6pt}
    u(x,y,t)&=-\cos(\pi x)\sin(\pi y)e^{-2\pi^2\nu t}\\
    v(x,y,t)&=\sin(\pi x)\cos(\pi y)e^{-2\pi^2\nu t}\\
    p(x,y,t)&=-\frac{\rho}{4}\left(\cos(2\pi x)+\cos(2\pi y)\right)e^{-4\pi^2\nu t}
\end{align*}
The specific instance of the above that we choose has viscosity ($\nu$) set to $10^{-2}$ and the density parameter ($\rho$) set to $1$.


A simple architecture has been chosen that satisfies Definition \ref{def:nn}, specifically $\gN_\vw(x, y, t)=(\mA_1\sigma(\mW\vz),\ip{\va_2}{\sigma(\mW_1\vz)})$, where $\vz=(x,y,t)\in\R^3,\mW\in\R^{p \times 3}$ and both $\sigma=\tanh$ and $\sigma=\tanh^3$ were tested. Each coordinate in $\mA_1\in\R^{3\times p}$  and $\va_2\in\R^p$ has been sampled from a Gaussian distribution with mean $0$ and variance $1$.

The space-time domain is set to the unit cube $(x,y,t)\in[0,1]^3$ and the loss function's parameters are chosen as $\lambda_0=1$ and $\lambda_1=0.3$ as defined in Definition \ref{def:lossclass}. The collocation points based constants, $C_\vz$ and $C_{\vz_0}$ are therefore $C_\vz=1,C_{\vz_0}=\frac{2}{3}$. The constants related to the activation function $\sigma$ ($B_\sigma, B_{\sigma'}, L_{\sigma}, L_{\sigma'}, L_{\sigma''}$) are set to $B_\sigma=1, B_{\sigma'}=1, L_{\sigma}=1, L_{\sigma'}=0.8, L_{\sigma''}=2, ~\text{for} ~\sigma=\tanh$ and $B_\sigma=1, B_{\sigma'}=0.75, L_{\sigma}=0.75, L_{\sigma'}=1.4, L_{\sigma''}=6 ~\text{for} ~\sigma=\tanh^3$.

The PINNs were trained for $20,000$ epochs (after which the training loss stabilized) using the AdamW optimizer as implemented in PyTorch with a learning rate of $1 \cdot 10^{-3}$. We varied the experimental setup by increasing the number of collocation points in the bulk ($N_r$) while keeping $N_0 = 2500$ fixed. This experiment was done for both $\tanh$ and $\tanh^3$ activations. The resulting generalization gap was measured for each trained net and  we also evaluated the Rademacher bound from Theorem~\ref{thm:NSRad} for each.

\begin{figure}[htbp!]
\centering
\begin{subfigure}
    \centering
    \includegraphics[width=0.4\linewidth]{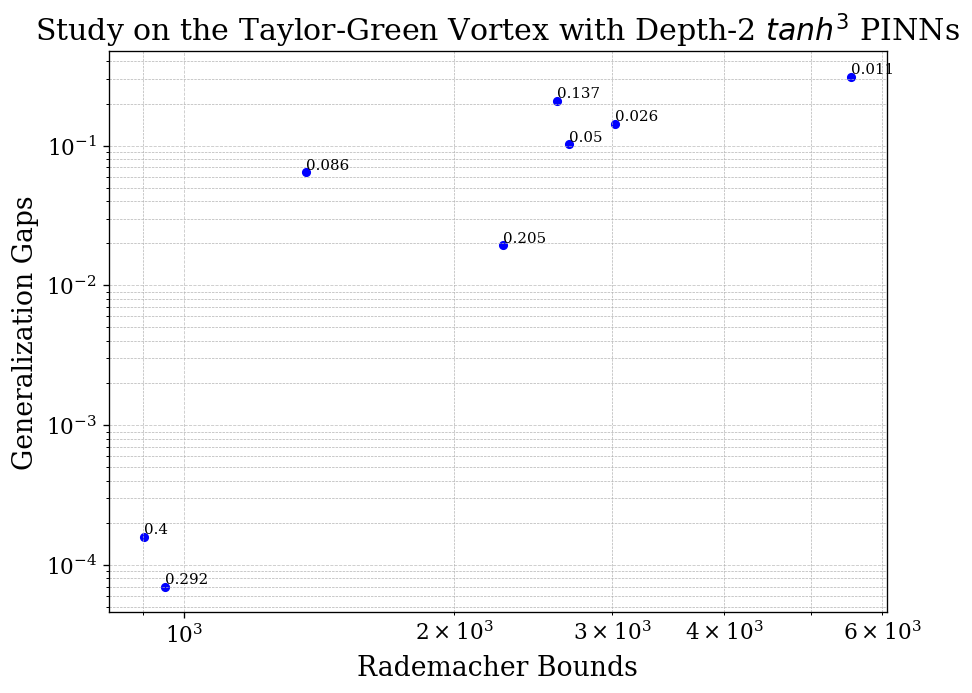}
\end{subfigure}
\begin{subfigure}
    \centering
    \includegraphics[width=0.4\linewidth]{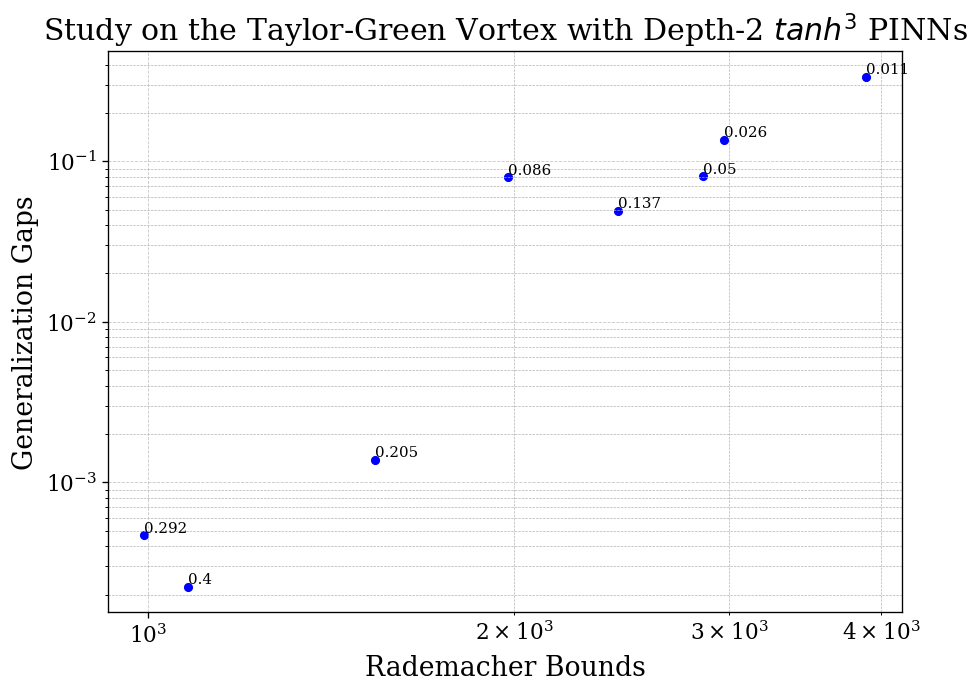}
\end{subfigure}
\caption{A scatter plot showing the relation between the generalization error upperbound given in Theorem \ref{thm:NSRad} and the generalization error calculated after training of the $tanh^3$ activated PINNs. Each point is labeled by the $N_r/N_0$ ratio at which the net was trained. On the left, the viscosity is set to $\nu=0.001$ and on the right $\nu=0.01$.}
\label{fig:TG_plot}
\end{figure}

In Figure \ref{fig:TG_plot} we give two plots of the measured generalization error of the $\tanh^3$ nets trained via the above experiment for $N_r=[3^3, 4^3, 5^3, 6^3, 7^3, 8^3, 9^3, 10^3]$. On the left, the viscosity was set to $\nu=0.001$ and on the right $\nu=0.01$. Across these variations of $N_r$  we compared the predicted bound from Theorem \ref{thm:NSRad} to the empirically derived gap. The measured correlation coefficient between these quantities were $0.905$ for $\nu = 0.001$ and $0.889$ for $\nu = 0.01$ -- {\em while we note that the training did not explicitly enforce the weight constraints via the $B_{f_s}$ bounds for $s \in [4]$, as specified in Definition~\ref{def:Wspace}.} For the same experiment using $\sigma=\tanh$ activation results in a correlation coefficient of $0.595$ and $0.687$ for the $\nu=0.001$ and $\nu=0.01$ respectively. Thus we evidence that our theoretical bound reproduces the observed scaling of performance with data. 
\section{Conclusion}\label{sec:conclusion}



We believe that the generalized error upper bound established here via Rademacher complexity extends to other non-linear PDEs whose non-linear terms are compatible with the contraction lemmas presented in this work. While this work establishes worst-case generalization bounds for the PINN risk of Navier-Stokes, deriving similar finite-sample guarantee for approximating the Navier-Stokes solution remains an open question. This gap stems from the fundamental fact that the generalization error of PINNs do not unconditionally ensure control over the error with respect to the true PDE solution. 

Furthermore, an important direction for future research is to extend such bounds as here for the Navier-Stokes equations to deeper neural networks, as has been demonstrated for linear PDEs in \cite{zheyuan22xpinn}.

Another promising avenue for future work is to establish lower bounds on the neural network size necessary to approximate the Navier–Stokes equations, under semi-supervised training, analogous to the results for the Hamilton-Jacobi-Bellman PDE in \cite{andresloan2025noisypde}.

State-of-the-art architectures such as AB-UPT \citep{alkin2025abupt} employ transformer-based models, for which no analogous theoretical analysis currently exist. An exciting direction for future work would be to extend to such data-assimilating transformer setups, such Rademacher analyses and the aforementioned lowerbound proofs for the semi-supervised setup.

Lastly, we note that a challenging direction for future work would be to establish analogous upper bounds that depend on the complexity of the underlying solution, as in Theorem 3.1 of \cite{zheyuan22xpinn} which was in the context of linear PDEs. Incorporating this aspect could further enable bounding the $\ell_2$ distance between the neural network approximation and the true solution, under appropriate assumptions on the differential operator, which is ultimately the main objective of such analyses.




\bibliography{sn-bibliography}

\appendix

\section{A Brief Introduction to the General PINN Risk}\label{sec:pinnbackground}

Consider the following specification of a PDE satisfied by an appropriately smooth function $\vu(\vx,t)$ 
\begin{align}\label{eq:pinnreview.pde}
    &\vu_t + \gN_\vx[\vu] = 0,&\vx \in D,\;\; t \in [0,T] \nonumber\\
    & \gC(\vu)(\vx,t) = 0 &\vx \in D,\;\; t \in [0,T] \nonumber\\
    &\vu(\vx,0) = \vf_0(\vx), &\vx \in D \\
    &\vu(\vx,t) = \vg(\vx,t), &\vx \in \partial D, \;\; t \in [0,T] \nonumber
\end{align}
where $\vx$ and $t$ represent the space and time dimensions, subscripts denote the partial differentiation variables, $\gN_\vx[\vu]$ is the (may be nonlinear) differential operator, $D$ is a subset of $\R^d$ s.t. it has a well-defined boundary $\partial D$. And $\gC$ is some functional in whose kernel the solution must like, like this could be the incompressibility condition in fluid dynamics.

Following \cite{raissi2019physics}, we try to approximate $\vu(\vx,t)$ by a deep neural network $\vu_\theta(\vx,t)$, and we define the corresponding residuals as,
\begin{gather}
\nonumber \gR_{\rm pde}(\vx,t) \coloneqq \partial_t \vu_\theta + \gN_\vx[\vu_\theta(\vx,t)], ~\gR_t(\vx) \coloneqq \vu_\theta(\vx,0) - \vf_0(\vx), \\ 
~\gR_b(\vx,t) \coloneqq \vu_\theta(\vx,t) - \vg(\vx,t), ~\gR_c(\vx,t) = \gC(\vu_\theta(\vx,t))
\end{gather}
Note that the partial derivative of the neural network ($\vu_\theta$) can be easily calculated using auto-differentiation facilities on modern deep-learning softwares.

Corresponding to a choice of three positive integers, $N_{\rm pde},N_0$ and $N_b$ denoting different counts of the number of points needed, typical forms of the losses constructed from the residuals mentioned above would be, 
\begin{align}
    \gL_{\rm pde} = \frac{1}{N_{\rm pde}} \sum_{i=1}^{N_{\rm pde}} \gR_{\rm pde}(\vx_r^i,t_r^i)^2, ~\gL_{0} = \frac{1}{N_{0}} \sum_{i=1}^{N_{0}} \gR_{t}(\vx_0^i)^2, ~\gL_{b} = \frac{1}{N_{b}} \sum_{i=1}^{N_{b}} \gR_{b}(\vx_b^i,t_b^i)^2, ~\gL_{\rm c} = \frac{1}{N_{\rm pde}} \sum_{i=1}^{N_{\rm pde}} \gR_{c}(\vx_r^i,t_r^i)^2
\end{align}
where $(\vx_r^i,t_r^i)$ denotes the points in the interior of $D \times [0,T]$, $(\vx_0^i)$ are the points sampled on the spatial domain $D$ and $(\vx_b^i,t_b^i)$ are the points sampled on the boundary of $D \times [0,T]$. 

Upon choosing regularizers $\lambda_0, \lambda_b, \lambda_c >0$, the neural net is then trained on an empirical loss function,
\begin{align}
    \hat{R}(\theta) \coloneqq \gL_{\rm pde}(\theta) + \lambda_0 \gL_0(\theta) + \lambda_b \gL_b(\theta) + \lambda_c \gL_c(\theta) 
\end{align}
where $\gL_{\rm pde},~\gL_0,~\gL_b$ and $\gL_c$ respectively penalize for $\gR_{\rm pde},~\gR_t,~\gR_b$ and $\gR_c$  being non-zero.

The common strategy here is to train a neural net $\vu_\theta$ such that $\gL(\theta)$ is as close to zero as possible. {\it It's important to note that except under specific conditions it remains largely unclear as to when this method of minimizing $\gL(\theta)$ would find a good approximation to the solution of the PDE in equation \ref{eq:pinnreview.pde}.}

Often the algorithms deployed to minimize the above empirical risk also seem to simultaneously also minimize the corresponding population risk function. If we denote by $\mu_r, \mu_0$ and $\mu_b$, the choices for the distribution on the spaces $D \times [0,T]$, $D$ and $[0,T]$ respectively then $\gL(\theta)$ above can be seen as an empirical estimate for the population risk, say $\gR(\theta)$ defined as,
\begin{gather}
    \gR(\theta) \\
    \notag \coloneqq \E_{(\vx_r,t_r) \sim \mu_r} [ \gR_{\rm pde}(\vx_r,t_r)^2 ] + \lambda_0  \E_{\vx_0 \sim \mu_0} [ \gR_{t}(\vx_0)^2 ]  + \lambda_b \E_{(\vx_b,t_b) \sim \mu_b} [ \gR_{b}(\vx_b,t_b)^2 ] + \lambda_c \E_{(\vx_r,t_r) \sim \mu_r} [ \gR_{c}(\vx_r,t_r)^2 ]
\end{gather}

{\it We note that it's a fundamental open question to be able to quantify the difference between $\hat{R}(\theta)$ and $\gR(\theta)$, in general. And in this work we take steps to answer this question for depth $2$ nets deployed for solving certain specifications of the Navier-Stokes PDE.}

\end{document}